\documentclass[twoside]{article}

\usepackage[preprint]{aistats2027}
\usepackage[round]{natbib}
\renewcommand{\bibname}{References}
\renewcommand{\bibsection}{\subsubsection*{\bibname}}

\usepackage{hyperref}
\usepackage{xurl}
\usepackage{mathtools}
\usepackage{amsmath, amssymb, amsthm}
\usepackage[ruled]{algorithm2e}
\usepackage{bm, dsfont}
\usepackage{bbm}
\usepackage{booktabs}
\usepackage{multirow}
\usepackage{caption}
\usepackage[labelformat=simple]{subcaption}
\usepackage{xcolor}
\usepackage{siunitx}
\newtheorem{definition}{Definition}
\newtheorem{theorem}{Theorem}
\newtheorem{lemma}{Lemma}

\newtheorem{corollary}{Corollary}

\mathtoolsset{showonlyrefs, showmanualtags}

\DeclareMathOperator{\cN}{\mathcal{N}}

\DeclareMathOperator{\vzero}{\boldsymbol{0}}

\begin{document}

\runningtitle{Generalized Deteminantal Point Processes for Attraction and Repulsion}
\twocolumn[

\aistatstitle{Generalized Determinantal Point Processes for Attraction and Repulsion with Distinct Interaction Kernels}

\aistatsauthor{Hideaki Ishibashi$^\ast$ \And Takahiro Kawashima$^\ast$ \And Hideitsu Hino }

\aistatsaddress{Kyushu Institute of Technology \And ZOZO Research \And Institute of Statistical Mathematics \\ Waseda University} ]

\begingroup
  \renewcommand\thefootnote{}
  \footnote{$^\ast$Equal contribution.}
\endgroup
\begin{abstract}
Determinantal point processes (DPPs) model subset selection probability with repulsion, but many applications also exhibit attraction arising from pairwise relations or shared latent structures. Nonsymmetric DPPs (NDPPs) capture both effects, but their kernels are difficult to interpret and specify. We propose the determinant-ratio point process (DRPP), which models repulsion and attraction via a ratio of the determinants of two symmetric positive-definite kernels. This formulation preserves the feature-based interpretation of standard DPPs while allowing the two effects to be specified separately. Furthermore, we introduce an auxiliary-variable representation that enables efficient MCMC sampling and yields tractable expectation formulas for key statistics including the normalizing constant, which together facilitate maximum-likelihood learning. We theoretically characterize the interaction structures well suited to DRPPs: for clustered attraction with $C$ latent groups, DRPPs exactly realize the target pairwise interactions with an attraction rank $C$, whereas NDPPs can require a rank that grows with the number of items given the same repulsion component. Conversely, nontransitive pair-specific attraction may require compensation through the DRPP repulsion kernel. Experiments support these predictions, showing that DRPPs fit transitive attraction with substantially fewer parameters than NDPPs.
\end{abstract}

\section{Introduction}

Determinantal point processes (DPPs)~\citep{Macchi1975,Kulesza2012} are probabilistic models for selecting low-redundancy subsets from a ground set, and they have been widely used in applications such as recommender systems~\citep{NEURIPS2018_dbbf603f}, document summarization~\citep{10.5555/3020548.3020597}, batch active learning~\citep{Biyik2019BatchAL}, and minibatch selection for stochastic gradient descent~\citep{NEURIPS2021_8744cf92}. Through the determinant of a positive semidefinite kernel matrix, a DPP assigns low probability to subsets of mutually similar items, thereby modeling repulsion between items.

However, many applications such as shopping basket completion also involve attraction, that is, a tendency for certain items to be selected together, which standard DPPs cannot express. Several models capable of expressing both attraction and repulsion have been proposed~\citep{kawashima2025,Zhang2021}, among which the nonsymmetric DPP (NDPP)~\citep{gartrell2019a} has been particularly widely studied as explained in the next section. An NDPP replaces the symmetric kernel of a DPP with a nonsymmetric matrix, which decomposes into a symmetric part and a skew-symmetric part; the latter gives rise to attraction. However, unlike the symmetric part, the skew-symmetric part cannot be interpreted as a Gram matrix of feature vectors. Consequently, it is difficult to specify the nonsymmetric matrix from prior knowledge and is typically learned from data.

In this work, we propose the determinant-ratio point process (DRPP), which models repulsion and attraction separately through a ratio of the determinants of two distinct kernel matrices. Items that are similar under the repulsion kernel are less likely to be selected together, as in a standard DPP, whereas items similar under the attraction kernel are more likely to be selected together. Because both kernels can be interpreted as similarities between items, attraction, like repulsion, can be specified directly from prior knowledge.

Although a DRPP has no closed-form normalizing constant, it retains many of the convenient properties of standard DPPs. Specifically, by introducing Gaussian latent variables, a DRPP can be represented as an infinite mixture of DPPs; conditioned on these latent variables, the distribution of subsets is exactly a DPP. This representation yields efficient, rejection-free Markov chain Monte Carlo (MCMC) samplers, namely, a block Gibbs sampler and a two-step sampler based on elliptical slice sampling~\citep{Murray2010}. Both samplers perform full-bit updates in a single Monte Carlo step. Furthermore, the resulting latent samples allow us to estimate DRPP statistics whose DPP counterparts have closed forms, such as the expected subset size, marginal probabilities, and the gradient of the log normalizing constant. Crucially, the estimated gradient facilitates maximum likelihood learning.

We further investigate, both theoretically and empirically, the interaction structures to which DRPPs are well suited. Theoretically, we show that when pairwise attraction arises from a small number of shared latent factors under uniform repulsion, a DRPP realizes the target interaction with an attraction rank equal to the number of factors. In contrast, an NDPP with the same repulsion component may require a skew-symmetric rank that grows with the size of the ground set. Conversely, the attraction modeled by a DRPP is inherently transitive: sufficiently strong attraction between items $i$ and $j$, and between $j$ and $k$, induces attraction between $i$ and $k$. Consequently, pair-specific, non-transitive attraction must be compensated for by the repulsion kernel, and we provide a lower bound on the required compensation. Experiments on synthetic and real data support these theoretical results. For transitive attraction, DRPPs achieve held-out likelihoods comparable to those of more flexible models with fewer parameters on synthetic data, and outperform NDPPs with fewer parameters on real data. 

\section{Related work}

A DPP on a ground set $\Omega=\{1,\dots,N\}$ is defined by an $N\times N$ symmetric positive semidefinite kernel matrix $\bm{L}$ as
\begin{equation}
\label{eq:dpp}
P(A; \bm{L})=\frac{\det(\bm{L}[A])}{\det(\bm{L}+\bm{I})}, \qquad \mathrm{for}~A\subseteq\Omega,
\end{equation}
where $\bm{L}[A]$ denotes the principal submatrix of $\bm{L}$ indexed by $A$ \citep{borodin2005}.
We also denote this $\mathrm{DPP}(A; \bm{L})$.
When $\bm{L}$ is the Gram matrix of feature vectors $\{\bm{\varphi}_i\}_{i\in\Omega}$, that is, $L_{ij}=\langle\bm{\varphi}_i,\bm{\varphi}_j\rangle$, $\det(\bm{L}[A])$ equals the squared volume spanned by $\{\bm{\varphi}_i\}_{i\in A}$~\citep{Kulesza2012}.
Writing $\bm{\varphi}_i=q_i\bm{b}_i$ with $\|\bm{b}_i\|_2=1$ yields the quality-diversity decomposition $\bm{L}=\bm{Q}\bm{R}\bm{Q}$, where $\bm{Q}=\mathrm{diag}(q_1,\dots,q_N)$ encodes the item quality and $\bm{R}=(\langle\bm{b}_i,\bm{b}_j\rangle)_{i,j\in\Omega}$ is a unit-diagonal similarity matrix.
Because the spanned volume shrinks as features become similar, similar items are less likely to be selected together. Consequently, standard DPPs can only model repulsion and cannot express positive correlations between items.

The nonsymmetric DPP (NDPP)~\citep{gartrell2019a} replaces the symmetric kernel with a nonsymmetric matrix consisting of a symmetric and a skew-symmetric part, allowing both attraction and repulsion. NDPPs retain tractable normalizing constants and marginals, enabling parameter learning~\citep{gartrell2019a,gartrell2020}, mode finding~\citep{gartrell2020,anari2022}, and sampling~\citep{poulson2020,han2022}. However, the skew-symmetric part lacks a natural Gram matrix interpretation, making it difficult to specify desired attraction strengths from prior knowledge. Recently, \citet{poinas2026} established well-definedness conditions for NDPPs and constructed couplings with intra-mark repulsion and inter-mark attraction, though how to control the range of attraction remains open.

Another line of work replaces the determinant defining DPP weights with other matrix functions to handle attraction and repulsion within a single family. \citet{Diaconis2000} proposed immanantal point processes, which use characters of the symmetric group to incorporate both the determinant (repulsion) and the permanent (attraction). $\alpha$-determinantal point processes~\citep{Shirai2003,Cunden2019} rely on the $\alpha$-determinant, whose scalar parameter $\alpha$ yields a broader family encompassing both determinants and permanents. Similarly, discrete kernel point processes (DKPPs)~\citep{kawashima2025} replace the determinant with a broader class of matrix functions, providing a unified family for attractive log-supermodular and repulsive log-submodular distributions. Since DKPPs use only symmetric kernel matrices, they can be specified directly from item features. However, all of these models govern attraction and repulsion globally, and thus cannot express mixed interaction structures with attraction between some pairs and repulsion between others.

Rather than switching between the determinant and the permanent, \citet{Katori2022} combined them in the permanental-determinantal point process, whose correlation functions are defined by the product of a permanent and a determinant. This product treats repulsion and attraction through separate factors, sharing our motivation of factorizing the two effects. However, computing the permanent is \#P-hard~\citep{Valiant1979}, and because both factors are built from a single kernel, repulsion and attraction are tied together. In contrast, we represent attraction through the inverse of a determinant, an idea also used in the determinantal clustering process~\citep{Shah2013}, which assigns $\det(\bm{L}[S])^{-\tau}$ to each block $S$ of a data partition to model within-cluster similarity. This model, however, deals only with attraction and does not consider repulsion.

Other approaches induce positive correlations by combining DPPs. Mixtures of low-rank DPPs~\citep{Gartrell2016}, introduced to improve expressive power, can induce positive dependence through heterogeneity across mixture components. Probabilistic generating circuits (PGCs)~\citep{Zhang2021} encode distributions via probability generating polynomials, with DPPs as special cases. By substituting other generating polynomials into the arguments of a DPP polynomial, the determinantal PGC (DetPGC) in the same work models arbitrary distributions within item groups while retaining DPP-based repulsion across groups, allowing positive dependence among within-group pairs. Although a DRPP can also be viewed as a continuous mixture of DPPs, its attraction is specified directly via an explicit attraction kernel, whereas in these approaches positive dependence arises only indirectly through component differences or within-group distributions, rather than being parameterized as a distinct interaction mechanism. 

\section{Determinant-Ratio Point Process (DRPP)}
\label{sec:def_drpp}

\subsection{Definition and Standard Form of DRPP}
On the ground set $\Omega$, we define the DRPP as follows.
\begin{definition}
Let $\bm{K}=(k_{ij})_{i,j\in\Omega}$ and $\bm{S}=(s_{ij})_{i,j\in\Omega}$ be $N\times N$ real symmetric positive definite matrices.
For each $A\subseteq\Omega$, define the unnormalized weight
\begin{equation}
\tilde{P}(A;\bm{K},\bm{S})=\frac{\det(\bm{K}[A])}{\det(\bm{S}[A])},
\label{eq:drpp}
\end{equation}
with $\det(\bm{K}[\emptyset])=\det(\bm{S}[\emptyset])=1$. The determinant-ratio point process (DRPP) is the distribution $P(A;\bm{K},\bm{S})=\tilde{P}(A;\bm{K},\bm{S})/Z(\bm{K},\bm{S})$, where $Z(\bm{K},\bm{S})=\sum_{A\subseteq\Omega}\tilde{P}(A;\bm{K},\bm{S})$ is the normalizing constant.
\end{definition}
The positive definiteness of $\bm{S}$ ensures that the denominator never vanishes, while that of $\bm{K}$ ensures that every subset has a positive probability. When $\bm{K}$ is the Gram matrix of item features, $\det(\bm{K}[A])$ is small for subsets containing mutually similar items under $\bm{K}$, so the numerator induces repulsion as in a standard DPP. Conversely, $\det(\bm{S}[A])$ is small for subsets of items similar under $\bm{S}$, so its inverse increases their probability and induces attraction. Setting $\bm{S}=\bm{I}$ recovers the standard DPP, and the DRPP thus extends the DPP by introducing attraction through $\bm{S}$.

Since $\det((\bm{D}\bm{K}\bm{D})[A])/\det((\bm{D}\bm{S}\bm{D})[A])=\det(\bm{K}[A])/\det(\bm{S}[A])$ for any positive diagonal matrix $\bm{D}$, the DRPP is invariant under a common diagonal scaling of $\bm{K}$ and $\bm{S}$. Thus, we can write, without loss of generality, $\bm{K}=\bm{Q}\bm{R}_K\bm{Q}$ and $\bm{S}=\bm{R}_S$, where $\bm{Q}=\mathrm{diag}(q_1,\dots,q_N)$ represents the item qualities, and $\bm{R}_K$ and $\bm{R}_S$ are positive definite matrices with unit diagonals.

In this work, we model $\bm{R}_K$ and $\bm{R}_S$ as
\begin{equation}
\bm{R}_K=\epsilon_K\bm{I}+\bm{V}\bm{V}^\top,\qquad\bm{R}_S=\epsilon_S\bm{I}+\bm{U}\bm{U}^\top,
\end{equation}
where $\bm{V}\in\mathbb{R}^{N\times r_K}$, $\bm{U}\in\mathbb{R}^{N\times r_S}$, and $0<\epsilon_K,\epsilon_S<1$. The $i$th rows $\bm{v}_i$ and $\bm{u}_i$ satisfy $\|\bm{v}_i\|_2^2=1-\epsilon_K$ and $\|\bm{u}_i\|_2^2=1-\epsilon_S$, ensuring that the unit-diagonal condition holds exactly.
Since $\bm{R}_K\succeq\epsilon_K\bm{I}$ and $\bm{R}_S\succeq\epsilon_S\bm{I}$, all principal submatrices are positive definite.
To remove the positivity constraint $q_i > 0$, we reparameterize item qualities as $\alpha_i=2\log q_i$, yielding the log unnormalized weights:
\begin{align}
&\log\tilde{P}(A;\bm{K},\bm{S})\\
&\quad=\sum_{i\in A}\alpha_i+\log\det(\bm{R}_K[A])-\log\det(\bm{R}_S[A]).
\label{eq:unnormalized_score}
\end{align}

\subsection{DRPP as Infinite Mixture of DPPs}
We derive an extended state space representation for DRPPs and show that a DRPP can be regarded as an infinite mixture of DPPs.
Since DPPs admit many closed-form expressions for key statistics, including expected cardinalities and marginal probabilities, this insight provides tractable expressions for these statistics under DRPPs.

Let us denote $\bm{M}^A \coloneq \epsilon^{-1}_S \sum_{j \in A} \bm{u}_j \bm{u}^\top_j$ and $\bm{L} \coloneq \epsilon_S^{-1} \bm{Q} \bm{R}_K \bm{Q}$, where $\bm{u}_j$ denotes the $j$th row vector of $\bm{U}$. Then, \eqref{eq:drpp} is written as
\begin{align}
    \label{eq:drpp_transformed}
    \tilde{P}(A) = \frac{\det(\bm{L} [A])}{\det(\bm{I} + \bm{M}^A)}.
\end{align}
Further, we introduce two random variables $\bm{z}_1, \bm{z}_2 \in \mathbb{R}^{r_S}$ and define the random kernel matrix
\begin{align}
    \label{eq:L_tilde}
    \tilde{\bm{L}}(\bm{z}_1, \bm{z}_2) \coloneq \bm{C}(\bm{z}_1, \bm{z}_2)^{\frac{1}{2}} \bm{L} \bm{C}(\bm{z}_1, \bm{z}_2)^{\frac{1}{2}},
\end{align}
where
\begin{align}
    \label{eq:C}
    \bm{C}(\bm{z}_1, \bm{z}_2) \coloneq \mathrm{diag}\left (\left \{ e^{-\frac{1}{2\epsilon_S}\sum^2_{\ell = 1} (\bm{u}^\top_i\bm{z}_\ell) ^2} \right \}_{i \in \Omega} \right ).
\end{align}

\begin{theorem}\label{thm:mixture}
    Let us consider the probability distribution 
    \begin{align}
        \label{eq:drpp_joint_1}
        P(A, \bm{z}_1, \bm{z}_2) 
        &\propto \det (\bm{L}[A])
         \prod^2_{\ell = 1}
         e^{-\frac{1}{2} \bm{z}^\top_\ell \bm{M}^A \bm{z}_\ell}
         \phi(\bm{z}_\ell)\\
        \label{eq:drpp_joint_2}
         &= \det (\tilde{\bm{L}}(\bm{z}_1, \bm{z}_2)[A]) \prod^2_{\ell = 1} \phi(\bm{z}_\ell),
    \end{align}
    where $\phi$ denotes the density of the $r_S$-variate standard Gaussian distribution.
    Then, marginalizing out $\bm{z}_1,\bm{z}_2$ yields the DRPP defined by~\eqref{eq:drpp_transformed}.
    Moreover, the marginal distribution over $\bm{z}_1$ and $\bm{z}_2$ is given by:
    \begin{align}
        \label{eq:Pz}
        P(\bm{z}_1, \bm{z}_2) \propto \det(\bm{I} + \tilde{\bm{L}}(\bm{z}_1, \bm{z}_2))
        \prod^2_{\ell = 1} \phi(\bm{z}_\ell).
    \end{align}
\end{theorem}
\begin{proof}
    Equation~\eqref{eq:drpp_joint_1} follows from the identity:
    \begin{align}
        \det(\bm{I} + \bm{M}^A)^{-1} = \mathbb{E}_{\bm{z}_1, \bm{z}_2 \sim \phi} \left [ \exp \left (-\frac{1}{2}  \sum^2_{\ell = 1} \bm{z}^\top_\ell \bm{M}^A \bm{z}_\ell \right ) \right].
    \end{align}
    Since the joint distribution \eqref{eq:drpp_joint_2} is a DPP with respect to $A$, summing over $A \subseteq \Omega$ yields \eqref{eq:Pz}.
\end{proof}
Conditioned on the auxiliary variables $(\bm{z}_1,\bm{z}_2)$, \eqref{eq:drpp_joint_2} is an ordinary DPP with the kernel $\tilde{\bm{L}}(\bm{z}_1, \bm{z}_2)$, since $P(A\mid \bm{z}_1, \bm{z}_2)\propto\det(\tilde{\bm{L}}(\bm{z}_1, \bm{z}_2)[A])$.
Thus, the probability function of the DRPP can be written as $P(A)=\int \mathrm{DPP}(A; \tilde{\bm{L}}(\bm{z}_1, \bm{z}_2))P(\bm{z}_1, \bm{z}_2)d\bm{z}_1 d\bm{z}_2$.
This leads to the key insight that a DRPP can be viewed as an infinite or continuous mixture of DPPs despite exhibiting attraction via $\bm{S}$.
Moreover, this property is highly advantageous for sampling, learning, and evaluating DRPPs.

\begin{corollary}
    \label{cor:Z_exp}
    The normalizing constant of the DRPP $Z = Z(\bm{K}, \bm{S})$ is
    \begin{align}
        Z = \mathbb{E}_{\bm{z}_1, \bm{z}_2 \sim \phi}[\det (\bm{I} + \tilde{\bm{L}}(\bm{z}_1, \bm{z}_2))].
    \end{align}
\end{corollary}
In practice, Corollary \ref{cor:Z_exp} provides a Monte Carlo estimator with reasonable variance for the normalizing constant of a DRPP when $\epsilon_S$ is not too small.
For small $\epsilon_S$, to reduce the variance, one can employ Chib's method \citep{chib1995} with the block Gibbs sampler introduced in Section \ref{sec:sampling}.

\begin{corollary}
    \label{cor:statistic}
    Suppose that $A$ is a random subset following $\mathrm{DRPP}(A; \bm{K}, \bm{S}) =  P(A; \bm{K}, \bm{S})$, and let $\tilde{\bm{K}}(\bm{z}_1, \bm{z}_2) \coloneq \tilde{\bm{L}}(\bm{z}_1, \bm{z}_2) (\bm{I} + \tilde{\bm{L}}(\bm{z}_1, \bm{z}_2))^{-1}$ denote the marginal kernel of $\mathrm{DPP}(\tilde{\bm{L}}(\bm{z}_1, \bm{z}_2))$.
    Then, for any statistic $g:2^{\Omega} \to \mathbb{R}$ satisfying $\mathbb{E}_{\mathrm{DRPP}(A; \bm{K}, \bm{S})}[|g(A)|] < \infty$, we have
    \begin{align}
        &\mathbb{E}_{\mathrm{DRPP}(A; \bm{K}, \bm{S})}[g(A)]\\
        \label{eq:statistic}
        &\quad = \mathbb{E}_{P(\bm{z}_1,\bm{z}_2)}
        \left[
            \mathbb E_{\mathrm{DPP}(\tilde{\bm{L}}(\bm{z}_1,\bm{z}_2))} [g(A) \mid \bm{z}_1, \bm{z}_2]
        \right ].
    \end{align}
\end{corollary}
Corollary \ref{cor:statistic} states that any integrable set statistic under the DRPP can be evaluated by averaging the corresponding conditional DPP expectation over the latent mixing distribution $P(\bm{z}_1,\bm{z}_2)$.
For some major statistics admitting analytical expectation formulas in DPPs, \eqref{eq:statistic} reduces to an expectation over $P(\bm{z}_1, \bm{z}_2)$.
For example, we have
\begin{align*}
    \mathbb{E}|A| &= \mathbb{E}_{P(\bm{z}_1,\bm{z}_2)}
    \left[
        \mathrm{tr} \tilde{\bm{K}}(\bm{z}_1,\bm{z}_2)
    \right],\\
    \mathrm{Var}|A|
    &=
    \mathbb{E}_{P(\bm{z}_1,\bm{z}_2)}
    \left[
        \mathrm{tr}\left( \tilde{\bm{K}}(\bm{z}_1, \bm{z}_2) - \tilde{\bm{K}}(\bm{z}_1, \bm{z}_2)^2 \right)
    \right] \\
    &\quad +
    \mathrm{Var}_{P(\bm{z}_1,\bm{z}_2)}
    \left[
        \mathrm{tr} \tilde{\bm{K}}(\bm{z}_1, \bm{z}_2)
    \right].
\end{align*}
Other statistics, such as marginal inclusion probabilities, can be obtained in similar expectation forms.

Assuming i.i.d.$\!$ samples can be drawn from $P(\bm{z}_1, \bm{z}_2)$, evaluating the inner conditional expectation analytically in \eqref{eq:statistic} corresponds to Rao-Blackwellization from a computational statistics perspective, thereby yielding low-variance Monte Carlo estimates.
Although generating exact i.i.d.$\!$ samples from $P(\bm{z}_1, \bm{z}_2)$ is generally intractable, we can construct efficient MCMC samplers for $P(\bm{z}_1, \bm{z}_2)$ and DRPPs, as discussed in the following section.

\begin{algorithm}[t]
    \SetKwComment{Comment}{//}{}
    \caption{Block Gibbs Sampler}\label{alg:gibbs}
    \KwIn{Repulsive kernel $\bm{K} = \bm{Q} \bm{R}_K \bm{Q}$, attractive kernel $\bm{S} = \epsilon_S \bm{I} + \bm{U} \bm{U}^\top$, and size constraint $k$ (optionally)}
    \KwOut{$A^1, \ldots, A^T$}
    Initialize $\bm{z}^0_1, \bm{z}^0_2 \in \mathbb{R}^{r_S}$\;
    \For{$t = 1$ to $T$}{
        Compute $\tilde{\bm{L}}(\bm{z}^{t-1}_1, \bm{z}^{t-1}_2)$ from \eqref{eq:L_tilde}\;
        \eIf{$k$ specified}{
            Sample $A^t \sim k {\text -} \mathrm{DPP}(\tilde{\bm{L}}(\bm{z}^{t-1}_1, \bm{z}^{t-1}_2))$\;
        }{
            Sample $A^t \sim \mathrm{DPP}(\tilde{\bm{L}}(\bm{z}^{t-1}_1, \bm{z}^{t-1}_2))$\;
        }
        Compute $\bm{M}^{A^t} = \epsilon^{-1}_S \sum_{j \in A^t}\bm{u}_j \bm{u}^\top_{j}$\;
        Sample $\bm{z}^t_\ell \sim \cN(\vzero, (\bm{I} + \bm{M}^{A^t})^{-1})$ for $\ell = 1,2$\;
    }
\end{algorithm}

\begin{algorithm}[t]
    \SetKwComment{Comment}{//}{}
    \caption{Two-step Sampler}\label{alg:2step}
    \KwIn{Repulsive kernel $\bm{K} = \bm{Q} \bm{R}_K \bm{Q}$, attractive kernel $\bm{S} = \epsilon_S \bm{I} + \bm{U} \bm{U}^\top$, and size constraint $k$ (optionally)}
    \KwOut{$A^1, \ldots, A^T$}
    Initialize $\bm{z}^0_1, \bm{z}^0_2 \in \mathbb{R}^{r_S}$\;
    Compute $\tilde{\bm{L}}(\bm{z}^0_1, \bm{z}^0_2)$ from \eqref{eq:L_tilde}\;
    \For{$t = 1$ to $T$}{
        \eIf{$k$ specified}{
            Sample $\bm{z}^t_1, \bm{z}^t_2$ from \eqref{eq:Pz_k} by ESS\;
            Update $\tilde{\bm{L}}(\bm{z}^t_1, \bm{z}^t_2)$\;
            Sample $A^t \sim k {\text -} \mathrm{DPP}(\tilde{\bm{L}}(\bm{z}^t_1, \bm{z}^t_2))$\;
        }{
            Sample $\bm{z}^t_1, \bm{z}^t_2$ from \eqref{eq:Pz} by ESS\;
            Update $\tilde{\bm{L}}(\bm{z}^t_1, \bm{z}^t_2)$\;
            Sample $A^t \sim \mathrm{DPP}(\tilde{\bm{L}}(\bm{z}^t_1, \bm{z}^t_2))$\;
        }
    }
\end{algorithm}

\subsection{Sampling}\label{sec:sampling}
Let us consider the joint distribution of the DRPP on the extended state space given in \eqref{eq:drpp_joint_1}, \eqref{eq:drpp_joint_2}.
This joint distribution leads to the following conditional distributions:
\begin{align}
    \label{eq:conditional_z}
    &P(\bm{z}_1, \bm{z}_2 \mid A) \propto \prod^2_{\ell = 1} \cN(\bm{z}_\ell; \vzero, (\bm{I} + \bm{M}^A)^{-1}),\\
    \label{eq:conditional_A}
    &P(A \mid \bm{z}_1, \bm{z}_2) = \mathrm{DPP}(A; \tilde{\bm{L}}(\bm{z}_1, \bm{z}_2)).
\end{align}
These conditional distributions yield an efficient block Gibbs sampler (Algorithm \ref{alg:gibbs}).

In addition, we can design a two-step sampling scheme for $A$,
in which we first sample from $P(\bm{z}_1, \bm{z}_2)$ and then sample from $P(A \mid \bm{z}_1, \bm{z}_2)$.
Since \eqref{eq:Pz} is the product of zero-mean Gaussians and a computable function of $\bm{z}_1, \bm{z}_2$, the elliptical slice sampler (ESS) \citep{Murray2010} can be employed to sample efficiently from $P(\bm{z}_1, \bm{z}_2)$.
The procedure of this two-step sampler is summarized in Algorithm \ref{alg:2step}.
Both samplers are rejection-free and perform full-bit updates in each Monte Carlo step.

The ESS for $P(\bm{z}_1, \bm{z}_2)$ is beneficial not only for sampling $A$ but also for evaluating the statistics in Corollary \ref{cor:statistic}.
In the next section, we show that the ESS also enables efficient maximum likelihood estimation of DRPPs via stochastic optimization.

\subsection{Learning DRPPs}
\label{sec:learning}
Let $\boldsymbol{\theta}$ denote the DRPP parameters to be learned; typically $\boldsymbol{\theta} = (\boldsymbol{\alpha}, \bm{V}, \bm{U})$.
Our aim is to maximize the regularized log-likelihood:
\begin{align}
\label{eq:mle}
\!\!\!\!\ell(\boldsymbol{\theta}) = \frac{1}{M}\sum^M_{m=1} \log \tilde{P}(A_m ; \boldsymbol{\theta}) - \log Z(\boldsymbol{\theta}) - \mathcal{R}(\boldsymbol{\theta}),
\end{align}
where $\mathcal{R}(\boldsymbol{\theta})$ denotes the regularization term.
The main difficulty in maximizing~\eqref{eq:mle} is evaluating $\nabla_{\boldsymbol{\theta}} \log Z(\boldsymbol{\theta})$.
Here, Theorem \ref{thm:mixture} plays a key role as well.

\begin{corollary}\label{cor:dZ}
    It holds that
    $\nabla_{\boldsymbol{\theta}} \log Z(\boldsymbol{\theta}) = \mathbb{E}_{\bm{z}_1, \bm{z}_2 \sim P(\bm{z}_1, \bm{z}_2)}[\nabla_{\boldsymbol{\theta}} \log \det (\bm{I} + \tilde{\bm{L}}(\bm{z}_1, \bm{z}_2))].$
\end{corollary}
\begin{proof}
    Let us denote $w_{\boldsymbol{\theta}}(\bm{z}_1, \bm{z}_2) \coloneq \det (\bm{I} + \tilde{\bm{L}}(\bm{z}_1, \bm{z}_2))$.
    From Corollary \ref{cor:Z_exp}, we have
    \begin{align}
        \nabla_{\boldsymbol{\theta}} Z(\boldsymbol{\theta}) =& \mathbb{E}_{\bm{z}_1, \bm{z}_2 \sim \phi}[\nabla_{\boldsymbol{\theta}}w_{\boldsymbol{\theta}}(\bm{z}_1, \bm{z}_2)] \\
        =& \mathbb{E}_{\bm{z}_1, \bm{z}_2 \sim \phi}[w_{\boldsymbol{\theta}}(\bm{z}_1, \bm{z}_2)\nabla_{\boldsymbol{\theta}} \log w_{\boldsymbol{\theta}}(\bm{z}_1, \bm{z}_2)].
    \end{align}
    Dividing this by $Z(\boldsymbol{\theta}) = \mathbb{E}_{\bm{z}_1, \bm{z}_2 \sim \phi}[w_{\boldsymbol{\theta}}(\bm{z}_1, \bm{z}_2)]$, we obtain
    \begin{align}
        &\nabla_{\boldsymbol{\theta}} \log Z(\boldsymbol{\theta})
        = \frac{\nabla_{\boldsymbol{\theta}} Z(\boldsymbol{\theta})}{Z(\boldsymbol{\theta})} \\
        &= \iint \nabla_{\boldsymbol{\theta}} \log w_{\boldsymbol{\theta}}(\bm{z}_1, \bm{z}_2) \frac{w_{\boldsymbol{\theta}}(\bm{z}_1, \bm{z}_2)\prod^2_{\ell=1} \phi(\bm{z}_\ell) d \bm{z}_\ell}{\iint w_{\boldsymbol{\theta}}(\bm{z}_1, \bm{z}_2)\prod^2_{\ell=1} \phi(\bm{z}_\ell) d \bm{z}_\ell}\\
        &= \mathbb{E}_{\bm{z}_1, \bm{z}_2 \sim P(\bm{z}_1, \bm{z}_2)}[\nabla_{\boldsymbol{\theta}} \log w_{\boldsymbol{\theta}}(\bm{z}_1, \bm{z}_2)],
    \end{align}
    which completes the proof.
\end{proof}
Using this corollary, we can estimate the gradient and train the parameters of a DRPP. The detailed algorithm is given in Appendix \ref{app:detail_algorithm}.

\subsection{DRPP with Fixed Cardinality}
\label{sec:fixed_cardinality}
For DPPs, there are many cases in which fixing the subset cardinality $k \in \Omega$ is more appropriate than allowing a variable subset size.
A DPP with a fixed cardinality $k$ is called a $k$-DPP, and the normalizing constant is given by $e_k(\bm{L})$, which is the $k$th order elementary symmetric polynomial of the eigenvalues of $\bm{L}$ \citep{kulesza2011}. We similarly refer to DRPPs with a fixed cardinality $k$ as $k$-DRPPs, to which the properties of $k$-DPPs naturally apply.

Let us denote the joint distribution corresponding to \eqref{eq:drpp_joint_1} with fixed $k = |A|$ as $P_k(A, \bm{z}_1, \bm{z}_2)$.
Regarding Theorem \ref{thm:mixture}, while the joint distribution $P_k(A, \bm{z}_1, \bm{z}_2)$ shares the same form as \eqref{eq:drpp_joint_2} up to a normalizing constant,
the marginal distribution over $(\bm{z}_1, \bm{z}_2)$ becomes
\begin{align}
    \label{eq:Pz_k}
    P(\bm{z}_1, \bm{z}_2; k) \propto e_k(\tilde{\bm{L}}(\bm{z}_1, \bm{z}_2)) \prod^2_{\ell=1}\phi(\bm{z}_\ell)
\end{align}
instead of \eqref{eq:Pz}.
For the block Gibbs sampler, the conditional distribution of $A$ becomes $P(A | \bm{z}_1, \bm{z}_2; k) = k\mbox{-}\mathrm{DPP}(\tilde{\bm{L}}(\bm{z}_1, \bm{z}_2))$ via \eqref{eq:conditional_A}.
These updates realize the ESS-based two-step sampler for $k$-DRPPs as well.

Corollaries \ref{cor:Z_exp} and \ref{cor:dZ} can be similarly adapted for $k$-DRPPs.
Letting $Z_k$ denote the normalizing constant of a $k$-DRPP, we have
\begin{align}
    Z_k &= \mathbb{E}_{\bm{z}_1, \bm{z}_2 \sim \phi}[e_k(\tilde{\bm{L}}(\bm{z}_1, \bm{z}_2))],\\
    \nabla \log Z_k &= \mathbb{E}_{\bm{z}_1, \bm{z}_2 \sim P(\bm{z}_1, \bm{z}_2; k)}[\nabla \log e_k(\tilde{\bm{L}}(\bm{z}_1, \bm{z}_2))].
\end{align}
Although $\nabla \log e_k(\tilde{\bm{L}}(\bm{z}_1, \bm{z}_2))$ has no tractable closed form, we can evaluate it via automatic differentiation. See Appendix \ref{app:detail_algorithm} for the detailed calculation.

\section{Comparison of DRPP and NDPP}
\label{sec:comp_drpp_ndpp}
Nonsymmetric DPPs (NDPPs) are existing models with a motivation very similar to ours \citep{gartrell2019a}.
NDPPs admit nonsymmetric kernel matrices for the DPP likelihood \eqref{eq:dpp}, thereby enabling the coexistence of attraction and repulsion.
To quantify the difference in behavior between DRPPs and NDPPs, we measure the pairwise interaction between items $i,j\in\Omega$ by
\begin{equation}
\Gamma_{ij}\coloneq\frac{P(\{i,j\})\,P(\emptyset)}{P(\{i\})\,P(\{j\})}=\frac{\tilde{P}(\{i,j\})\,\tilde{P}(\emptyset)}{\tilde{P}(\{i\})\,\tilde{P}(\{j\})}.
\label{eq:pairwise_interaction}
\end{equation}
We call $i$ and $j$ attractive if $\Gamma_{ij}>1$, repulsive if $\Gamma_{ij}<1$, and neutral if $\Gamma_{ij}=1$. We can calculate $\Gamma_{ij}$ using the unnormalized weights of $\emptyset$, $\{i\}$, $\{j\}$ and $\{i,j\}$, as the normalizing constants cancel out.
For a DRPP, $\Gamma_{ij}$ is represented as a ratio of repulsive and attractive terms (as shown in Lemma~\ref{lem:drpp_pairwise} of Appendix~A):
\begin{equation}
    \Gamma_{ij}=\frac{1-(\rho^K_{ij})^2}{1-(\rho^S_{ij})^2},
    \label{eq:drpp_pair}
\end{equation}
where $\rho^K_{ij}$ and $\rho^S_{ij}$ are the $(i,j)$-components of $\bm{R}_K$ and $\bm{R}_S$, respectively, corresponding to the correlation coefficients under $\bm{K}$ and $\bm{S}$. Note that $\Gamma_{ij}$ is independent of $\bm{Q}$.

For an NDPP, the nonsymmetric matrix $\bm{L}$ is uniquely decomposed as $\bm{L}=\bm{H}+\bm{A}$ with a symmetric matrix $\bm{H}$ and a skew-symmetric matrix $\bm{A}$, yielding $\Gamma_{ij}=1-(H_{ij}^2-A_{ij}^2)/(H_{ii}H_{jj})$ (see Lemma~\ref{lem:ndpp_pairwise} of Appendix~A).
These matrices are typically learned with the low-rank parameterizations $\bm{H}=\bm{F}\bm{F}^\top$ and $\bm{A}=\bm{B}\bm{C}^\top-\bm{C}\bm{B}^\top$, where $\bm{F}\in\mathbb{R}^{N\times r_H}$ and $\bm{B},\bm{C}\in\mathbb{R}^{N\times r_A}$~\citep{gartrell2019a}.
While attraction can take various forms, we focus on two simple structures, clustered attraction and pair-specific attraction, to highlight the differences in the properties of NDPP and DRPP. 

\paragraph{Clustered attraction.}
Suppose that $\Omega$ is partitioned into $C$ nonempty clusters, $\Omega=\Omega_1\sqcup\cdots\sqcup\Omega_C$ with $|\Omega_c|=m_c$, and that the pairwise interactions under the data-generating distribution are given by
\begin{equation}
\Gamma^\ast_{ij}=
\begin{cases}
\gamma_+ & \text{if } i,j\in\Omega_c \text{ for some } c,\\
\gamma_- & \text{if } i\in\Omega_c,\ j\in\Omega_{c'},\ c\neq c',
\end{cases}
\end{equation}
where $0<\gamma_-\leq 1<\gamma_+$.
This corresponds to a uniform block structure with attraction within clusters and repulsion or neutrality between clusters.
\begin{theorem}
\label{thm:block}
Let $\kappa=\sqrt{1-\gamma_-}$, and assume the same repulsion component for both models, $\bm{K}=\bm{H}=(1-\kappa)\bm{I}+\kappa\mathbf{1}\mathbf{1}^\top$.
Then the following hold.

\textup{(i)} There exists a DRPP with $\bm{U}\in\mathbb{R}^{N\times C}$ (i.e., $r_S=C$) that satisfies $\Gamma_{ij}=\Gamma^\ast_{ij}$ for all $i\neq j$.

\textup{(ii)} Any NDPP with $\bm{B},\bm{C}\in\mathbb{R}^{N\times r_A}$ that satisfies $\Gamma_{ij}=\Gamma^\ast_{ij}$ for all $i\neq j$ requires
\begin{equation}
    r_A\geq\sum_{c=1}^{C}\left\lfloor\frac{m_c}{2}\right\rfloor\geq\frac{N-C}{2}.
\end{equation}

\end{theorem}
\begin{proof}
The proof is given in Appendix~A.
\end{proof}
Thus, for a fixed number of clusters $C$, a DRPP realizes $\Gamma^\ast$ with rank $C$, whereas the rank required for the skew-symmetric part of an NDPP grows linearly with $N$. Note that this comparison fixes the symmetric part of the NDPP to the shared repulsion component. The attraction in this setting is transitive: if $i$ attracts $j$ and $j$ attracts $k$, then $i$, $j$, and $k$ belong to the same cluster, and thus $i$ also attracts $k$.

\paragraph{Pair-specific attraction.}
We next examine the two models in settings where attraction is not necessarily transitive.
Let $G=(\Omega,E)$ be a simple undirected graph representing attraction, and assign a strength $\gamma_{ij}>1$ to each edge $(i,j)\in E$.
The target pairwise interactions are defined as
\begin{equation}
\Gamma^\ast_{ij}=
\begin{cases}
\gamma_{ij} & \text{if } (i,j)\in E,\\
1 & \text{if } (i,j)\notin E
\end{cases}
\label{eq:pairwise_attraction}
\end{equation}

\begin{theorem}
\label{thm:skew}
For $\Gamma^\ast$ in \eqref{eq:pairwise_attraction}, the following hold.

\textup{(i)} For any graph $G$ and any strengths $\{\gamma_{ij}\}_{(i,j)\in E}$, there exists an NDPP with $\bm{L}=\bm{I}+\bm{A}$ that satisfies $\Gamma_{ij}=\Gamma^\ast_{ij}$ for all $i\neq j$.

\textup{(ii)} Let $i,j,k\in\Omega$ be distinct items with $(i,j),(j,k)\in E$ and $(i,k)\notin E$. If a DRPP satisfies $\rho^S_{ij},\rho^S_{jk}\geq0$, $\Gamma_{ij}\geq\gamma$, $\Gamma_{jk}\geq\gamma$, and $|\rho^K_{ik}|\leq\delta$ for some $\gamma>2$ and $\delta\geq0$, then
\begin{equation}
\Gamma_{ik}\geq\frac{\gamma^2}{4(\gamma-1)}\left(1-\delta^2\right).
\label{eq:endpoint_bound}
\end{equation}
\end{theorem}
\begin{proof}
    The proof is given in Appendix~A.
\end{proof}

Equation \eqref{eq:endpoint_bound} implies that $\Gamma_{ik}>1$ whenever $\delta<1-2/\gamma$.
Hence, to realize $\Gamma^\ast$ with $\gamma_{ij}=\gamma_{jk}=\gamma>2$ and $\Gamma^\ast_{ik}=1$ using a DRPP, the repulsion kernel must satisfy
\begin{equation}
|\rho^K_{ik}|\geq1-\frac{2}{\gamma}.
\label{eq:cancellation}
\end{equation}
That is, a DRPP can express non-transitive attraction, but only via sufficient cancellation from the repulsion kernel. In particular, when $\bm{K}=\bm{I}$, meaning $\delta=0$, no non-transitive interaction pattern with $\gamma>2$ can be realized.

In summary, a DRPP represents attraction through inner products of low-dimensional latent vectors, making it well suited to attraction arising from shared factors. Conversely, an NDPP can specify the entries of its skew-symmetric matrix pair by pair when its rank is unconstrained, making it better suited to pair-specific attraction. Since real-world data may exhibit both types of structures, our experiments evaluate them separately.

\section{Experimental results}
\label{sec:exp}

\begin{figure*}[th]
    \centering
    \begin{minipage}[t]{0.30\linewidth}
        \centering
        (a) continuous
        \includegraphics[width=\linewidth]{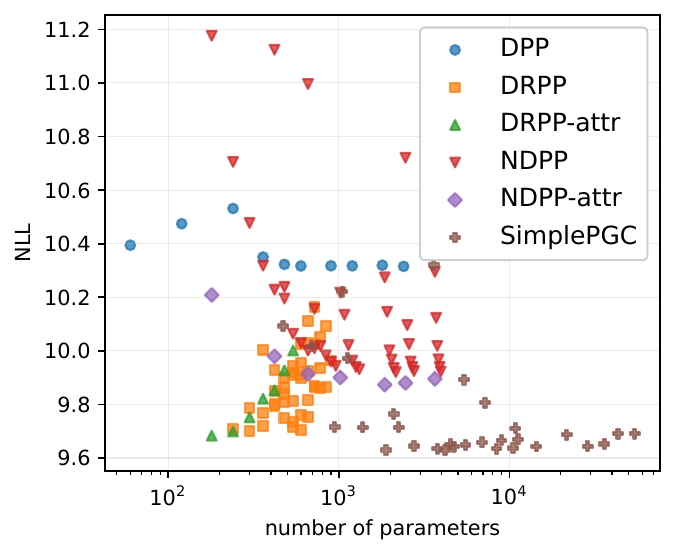}\\
    \end{minipage}\hfill
    \begin{minipage}[t]{0.30\linewidth}
        \centering
        (b) cluster
        \includegraphics[width=\linewidth]{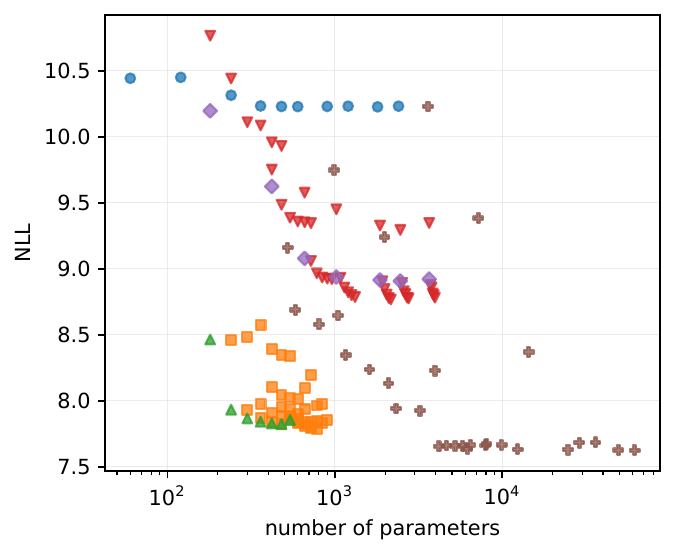}\\
    \end{minipage}\hfill
    \begin{minipage}[t]{0.30\linewidth}
        \centering
        (c) skew
        \includegraphics[width=\linewidth]{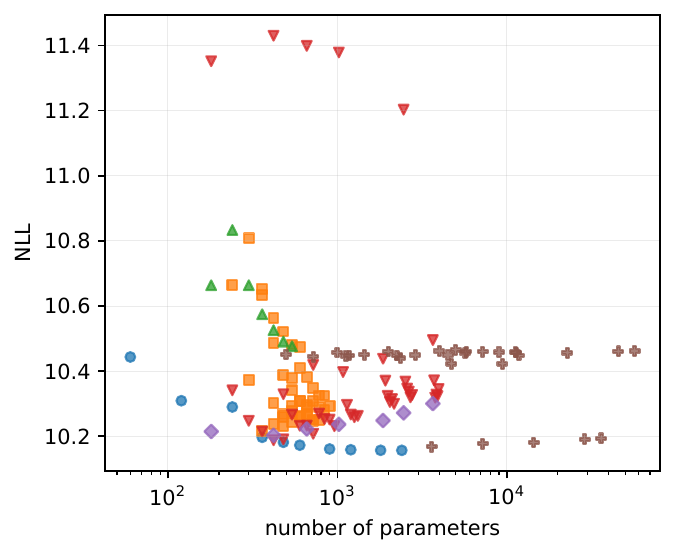}\\
    \end{minipage}
    
    
    \begin{minipage}[t]{0.30\linewidth}
        \centering
        (d) Amazon
        \includegraphics[width=\linewidth]{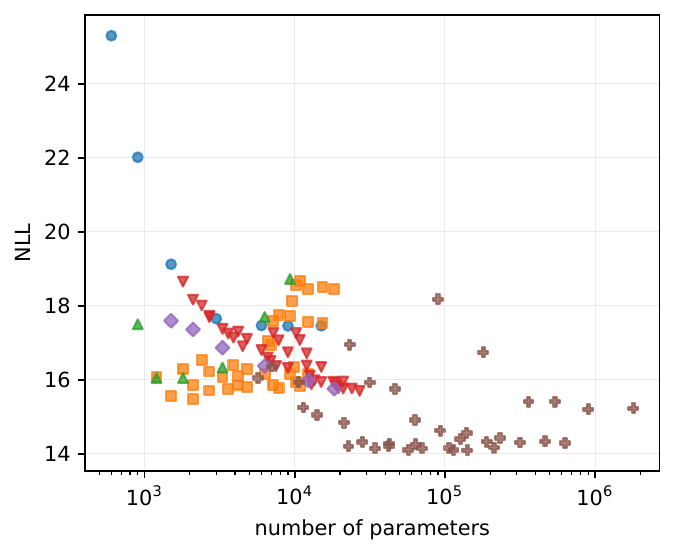}\\
    \end{minipage}\hfill
    \begin{minipage}[t]{0.30\linewidth}
        \centering
        (e) KDD
        \includegraphics[width=\linewidth]{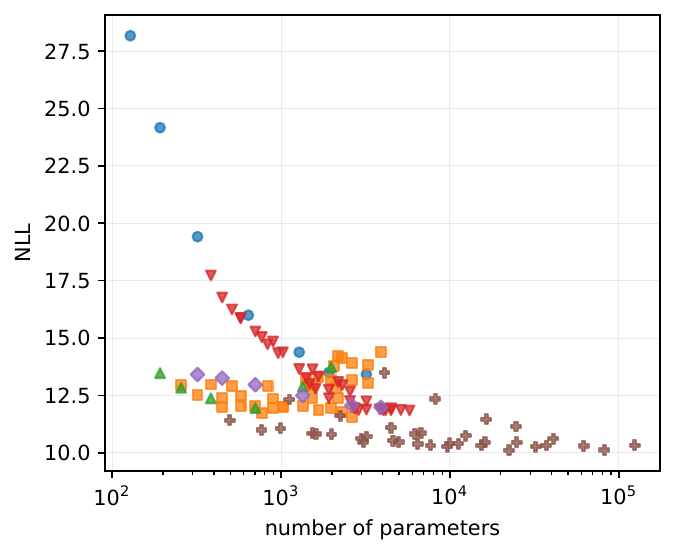}\\
    \end{minipage}\hfill
    \begin{minipage}[t]{0.30\linewidth}
        \centering
        (f) COCO
        \includegraphics[width=\linewidth]{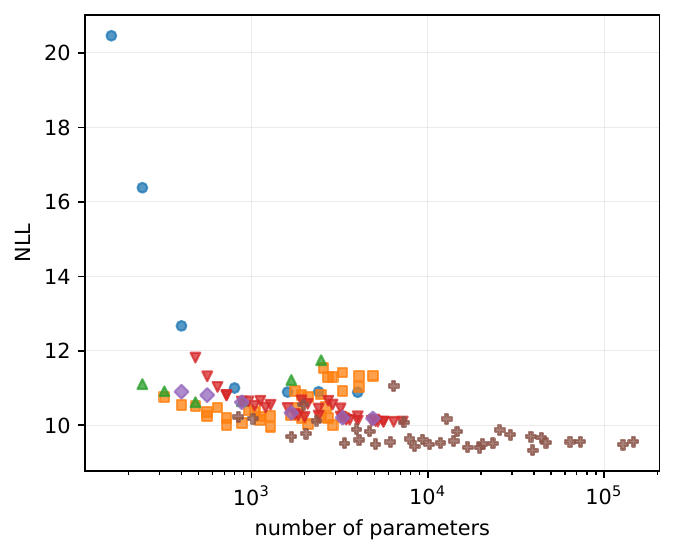}\\
    \end{minipage}
    \caption{Number of parameters vs.\ validation NLL (lower is better) on the synthetic datasets (a)--(c) and the real datasets (d)--(f). Each point corresponds to one hyperparameter setting.}
    \label{fig_complexity}
\end{figure*}

Based on the analysis in Section~\ref{sec:comp_drpp_ndpp}, DRPPs are expected to represent transitive attraction with few parameters, whereas their advantage should diminish for non-transitive attraction, which requires cancellation by the repulsion kernel. In this section, we examine these properties on synthetic and real data.

\subsection{Experimental Setup}
\paragraph{Evaluation and baselines.}
For each method, we evaluate the validation negative log-likelihood (NLL) against the number of parameters while varying the hyperparameters. The number of parameters is defined as the total size of the trainable real-valued arrays. We compare DRPPs against NDPPs, standard DPPs, and SimplePGC. We also include attraction-only variants of DRPPs and NDPPs, in which the repulsion component $\bm{Q}\bm{R}\bm{Q}$ has its correlation matrix $\bm{R}$ fixed to the identity and only the item qualities $\bm{Q}$ are learned (i.e., $\bm{K}=\bm{Q}^2$ for DRPPs and $\bm{H}=\bm{Q}^2$ for NDPPs). We denote them as DRPP-attr and NDPP-attr. The NLL of DRPPs is computed exactly on synthetic data and approximated using the ESS on real data. See Appendix~\ref{app:detail_exp} for candidate hyperparameters.

\paragraph{Datasets.}
We use three types of pairwise interactions on $N=60$ items: cluster, continuous, and skew types (see Appendix~\ref{app:detail_exp} for the definitions). In the cluster type, attraction acts uniformly within each of six clusters, and pairs in different clusters are neutral. In the continuous type, items are placed on a quarter circle, and items with closer angles attract each other more strongly. Both types are transitive and can be represented exactly by the attraction kernel of a DRPP. The skew type is a non-transitive structure over three groups, in which pairs within the same group are neutral and pairs in different groups attract each other. It can be represented by the skew-symmetric part of an NDPP. To avoid the effects of higher-order interactions, we generated data from the pairwise log-linear model $p(A\mid|A|=k)\propto\prod_{i<j\in A}\Gamma_{ij}$ with the subset size fixed to $k=3$. We also generated \num{10000} validation subsets separately from the training data. Since the subset size is fixed, all methods are trained by maximizing the likelihood of the distribution conditioned on the size $k$.

As real data, we use the three-category data from Amazon Baby Registries~\citep{gartrell2019a} (hereafter Amazon) and KDD Cup 2000 (hereafter KDD) from Twenty Datasets~\citep{Haaren2012}, which have been used in previous studies~\citep{gartrell2019a, Zhang2021}, and the object co-occurrence data of Microsoft~COCO~\citep{Lin2014} (hereafter COCO). Amazon, KDD, and COCO exhibit substantial excess transitivity (see Appendix~\ref{app:detail_exp}) and are used to examine whether the parameter efficiency observed in the synthetic transitive settings also appears in real data. As in \citet{Gartrell2016}, we removed subsets with fewer than two items from all real datasets.
In our setting, such subsets carry no information about item interactions and would otherwise dominate the likelihood.

\begin{figure}[th]
    \centering
    \begin{minipage}{0.49\columnwidth}
        \centering
        (a) DRPP-attr\\
        \includegraphics[width=1.0\linewidth]{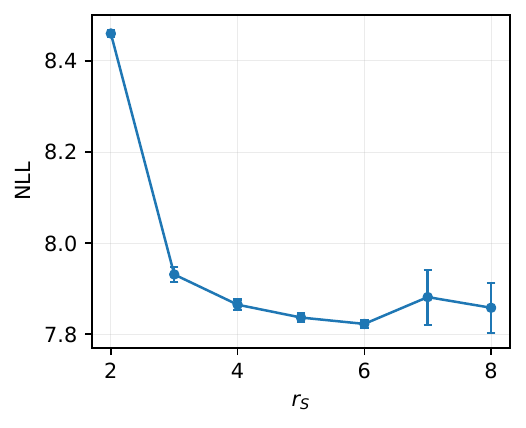}
    \end{minipage}
    \begin{minipage}{0.49\columnwidth}
        \centering
        (b) NDPP-attr\\
        \includegraphics[width=1.0\linewidth]{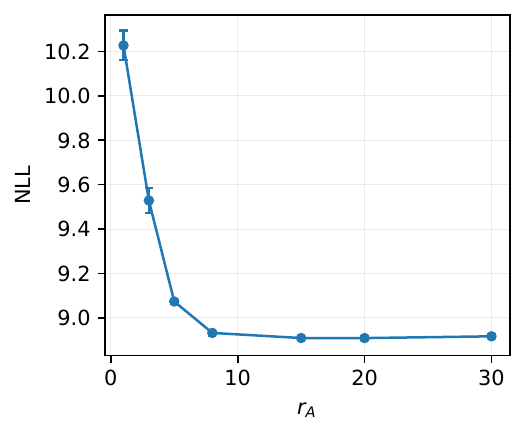}
    \end{minipage}
    \caption{Comparison between the rank of DRPPs and that of NDPPs on cluster dataset. Error bars indicate mean $\pm$ standard deviation over five independent runs.}
    \label{fig_comp_drpp_ndpp}
\end{figure}

\subsection{Parameter Efficiency}
Figure~\ref{fig_complexity} shows the results. On the transitive cluster and continuous types, DRPPs achieve NLL comparable to SimplePGC using several to about ten times fewer parameters, and achieve lower NLL than NDPPs with fewer parameters. In particular, on the cluster type, NDPPs do not match the NLL of DRPPs across any of the hyperparameter settings we explored. Indeed, as shown in Figure~\ref{fig_comp_drpp_ndpp}, when varying $r_S$ of DRPP-attr and $r_A$ of NDPP-attr with the repulsion fixed, the validation NLL of DRPPs is minimized at $r_S=C=6$, equal to the number of clusters. In contrast, the NLL of NDPPs saturates as $r_A$ increases and fails to reach that of DRPPs even beyond $r_A=30$, which Theorem~\ref{thm:block} requires. This is consistent with the prediction of Theorem~\ref{thm:block} that DRPPs can represent the target interactions with a rank equal to the number of clusters, whereas NDPPs cannot do so easily. On the skew type, the clear parameter-efficiency advantage of DRPPs disappears, with the best NLLs of DRPPs, NDPPs, and DPPs are similar. This is consistent with the prediction in Section 4 that the advantage of DRPPs diminishes for non-transitive attraction.

On real data, DRPPs achieve NLL comparable to or better than that of NDPPs with an order of magnitude fewer parameters on Amazon, occupying the low-parameter end of the Pareto front between the number of parameters and NLL. SimplePGC achieves the lowest NLL but requires more than an order of magnitude more parameters. On KDD and COCO, DRPPs achieve lower NLL than NDPPs with fewer parameters, though SimplePGC obtains lower NLL than DRPPs even with a comparable number of parameters.

As a remark, while SimplePGC consistently achieves lower NLL than DRPPs across datasets, this is consistent with the greater flexibility of SimplePGC. While a DRPP is defined as an infinite mixture of DPPs induced by just two kernel matrices, SimplePGC encompasses arbitrary finite mixtures of DPPs as special cases.
Therefore, a well-tuned SimplePGC is best viewed as a flexible reference baseline rather than a direct competitor of DRPPs. 

\section{Conclusion}
\label{sec:conclusion}
We proposed the DRPP, which models repulsion and attraction through the ratio of the determinants of two feature-based symmetric positive-definite kernels. By representing a DRPP as an infinite mixture of DPPs via auxiliary Gaussian variables, we enabled efficient MCMC sampling, tractable estimation of key set statistics, and maximum-likelihood learning. Theoretical and empirical analyses demonstrated that DRPPs represent transitive and clustered attraction with high parameter efficiency compared to NDPPs, whereas non-transitive attraction requires cancellation through the repulsion kernel.

While this strong structural inductive bias enables high parameter efficiency, it also constrains flexibility relative to broader models like SimplePGC. Extending DRPPs by replacing the underlying DPP components with more expressive building blocks (e.g., DetPGC) to capture both transitive and pair-specific attraction while retaining latent-variable inference is an intriguing direction for future work. While we did not address DRPP mode finding in this work, difference-of-submodular optimization \citep{iyer2012} serves as the primary option. However, such methods lack formal approximation guarantees, making the design of algorithms with provable bounds an important open challenge. Other valuable directions include reducing the variance of the ESS-based estimator for large subset sizes and extending our theoretical analysis to general NDPP lower bounds and higher-order interactions.

\subsubsection*{Acknowledgements}
Part of this work is supported by JST CREST JPMJCR26E3, JSPS KAKENHI JP24K15088, JP25H01494, JP26K02989 and JP26K23861.

\bibliography{aistats2027}

\begin{thebibliography}{}

\bibitem[Anari and Vuong, 2022]{anari2022}
Anari, N. and Vuong, T.-D. (2022).
\newblock From {{Sampling}} to {{Optimization}} on {{Discrete Domains}} with {{Applications}} to {{Determinant Maximization}}.
\newblock In {\em Proceedings of {{Thirty Fifth Conference}} on {{Learning Theory}}}, pages 5596--5618. PMLR.

\bibitem[Anscombe, 1956]{Anscombe1956}
Anscombe, F.~J. (1956).
\newblock On estimating binomial response relations.
\newblock {\em Biometrika}, 43(3-4):461--464.

\bibitem[Anzai and Hino, 2026]{anzai2026}
Anzai, S. and Hino, H. (2026).
\newblock Determinantal {{Point Process Approximation}} under {{Positive}} and {{Negative Dependence}}.

\bibitem[Bardenet et~al., 2021]{NEURIPS2021_8744cf92}
Bardenet, R., Ghosh, S., and LIN, M. (2021).
\newblock Determinantal point processes based on orthogonal polynomials for sampling minibatches in sgd.
\newblock In Ranzato, M., Beygelzimer, A., Dauphin, Y., Liang, P., and Vaughan, J.~W., editors, {\em Advances in Neural Information Processing Systems}, volume~34, pages 16226--16237. Curran Associates, Inc.

\bibitem[Biyik et~al., 2019]{Biyik2019BatchAL}
Biyik, E., Wang, K., Anari, N., and Sadigh, D. (2019).
\newblock Batch active learning using determinantal point processes.
\newblock {\em ArXiv}, abs/1906.07975.

\bibitem[Borodin and Rains, 2005]{borodin2005}
Borodin, A. and Rains, E.~M. (2005).
\newblock Eynard--{{Mehta Theorem}}, {{Schur Process}}, and {{Their Pfaffian Analogs}}.
\newblock {\em Journal of Statistical Physics}, 121(3):291--317.

\bibitem[Chen et~al., 2018]{NEURIPS2018_dbbf603f}
Chen, L., Zhang, G., and Zhou, E. (2018).
\newblock Fast greedy map inference for determinantal point process to improve recommendation diversity.
\newblock In Bengio, S., Wallach, H., Larochelle, H., Grauman, K., Cesa-Bianchi, N., and Garnett, R., editors, {\em Advances in Neural Information Processing Systems}, volume~31. Curran Associates, Inc.

\bibitem[Chib, 1995]{chib1995}
Chib, S. (1995).
\newblock Marginal {{Likelihood}} from the {{Gibbs Output}}.
\newblock {\em Journal of the American Statistical Association}, 90(432):1313--1321.

\bibitem[Cunden et~al., 2019]{Cunden2019}
Cunden, F.~D., Majumdar, S.~N., and O’Connell, N. (2019).
\newblock Free fermions and $\alpha$-determinantal processes.
\newblock {\em Journal of Physics A: Mathematical and Theoretical}, 52(16):165202.

\bibitem[Diaconis and Evans, 2000]{Diaconis2000}
Diaconis, P. and Evans, S.~N. (2000).
\newblock Immanants and finite point processes.
\newblock {\em J. Comb. Theory A}, 91:305--321.

\bibitem[Gartrell et~al., 2019]{gartrell2019a}
Gartrell, M., Brunel, V.-E., Dohmatob, E., and Krichene, S. (2019).
\newblock Learning {{Nonsymmetric Determinantal Point Processes}}.
\newblock In {\em Advances in {{Neural Information Processing Systems}}}, volume~32. Curran Associates, Inc.

\bibitem[Gartrell et~al., 2020]{gartrell2020}
Gartrell, M., Han, I., Dohmatob, E., Gillenwater, J., and Brunel, V.-E. (2020).
\newblock Scalable {{Learning}} and {{MAP Inference}} for {{Nonsymmetric Determinantal Point Processes}}.
\newblock In {\em International {{Conference}} on {{Learning Representations}}}.

\bibitem[Gartrell et~al., 2016]{Gartrell2016}
Gartrell, M., Paquet, U., and Koenigstein, N. (2016).
\newblock The bayesian low-rank determinantal point process mixture model.
\newblock {\em ArXiv}, abs/1608.04245.

\bibitem[Haldane, 1956]{Haldane1956}
Haldane, B. J. B.~S. (1956).
\newblock The estimation and significance of the logarithm of a ratio of frequencies.
\newblock {\em Annals of Human Genetics}, 20.

\bibitem[Han et~al., 2022]{han2022}
Han, I., Gartrell, M., Dohmatob, E., and Karbasi, A. (2022).
\newblock Scalable {{MCMC Sampling}} for {{Nonsymmetric Determinantal Point Processes}}.
\newblock In {\em Proceedings of the 39th {{International Conference}} on {{Machine Learning}}}, pages 8213--8229. PMLR.

\bibitem[Iyer and Bilmes, 2012]{iyer2012}
Iyer, R. and Bilmes, J. (2012).
\newblock Algorithms for approximate minimization of the difference between submodular functions, with applications.
\newblock In {\em Proceedings of the {{Twenty-Eighth Conference}} on {{Uncertainty}} in {{Artificial Intelligence}}}, {{UAI}}'12, pages 407--417. AUAI Press.

\bibitem[Katori and Shirai, 2022]{Katori2022}
Katori, M. and Shirai, T. (2022).
\newblock Zeros of the i.i.d. gaussian laurent series on an annulus: Weighted szeg{\H o}kernels and permanental-determinantal point processes.
\newblock {\em Communications in Mathematical Physics}, 392(3):1099--1151.

\bibitem[Kawashima and Hino, 2025]{kawashima2025}
Kawashima, T. and Hino, H. (2025).
\newblock A {{Family}} of {{Distributions}} of {{Random Subsets}} for {{Controlling Positive}} and {{Negative Dependence}}.
\newblock In {\em Proceedings of {{The}} 28th {{International Conference}} on {{Artificial Intelligence}} and {{Statistics}}}, pages 64--72. PMLR.

\bibitem[Kulesza and Taskar, 2011a]{kulesza2011}
Kulesza, A. and Taskar, B. (2011a).
\newblock K-{{Dpps}}: {{Fixed-Size Determinantal Point Processes}}.
\newblock In {\em International {{Conference}} on {{Machine Learning}}}.

\bibitem[Kulesza and Taskar, 2011b]{10.5555/3020548.3020597}
Kulesza, A. and Taskar, B. (2011b).
\newblock Learning determinantal point processes.
\newblock In {\em Proceedings of the Twenty-Seventh Conference on Uncertainty in Artificial Intelligence}, UAI'11, page 419^^e2^^80^^93427, Arlington, Virginia, USA. AUAI Press.

\bibitem[Kulesza and Taskar, 2012]{Kulesza2012}
Kulesza, A. and Taskar, B. (2012).
\newblock Determinantal {{Point Processes}} for {{Machine Learning}}.
\newblock {\em Foundations and Trends\textregistered{} in Machine Learning}, 5(2-3):123--286.

\bibitem[Lin et~al., 2014]{Lin2014}
Lin, T.-Y., Maire, M., Belongie, S., Hays, J., Perona, P., Ramanan, D., Doll{\'a}r, P., and Zitnick, C.~L. (2014).
\newblock Microsoft coco: Common objects in context.
\newblock In Fleet, D., Pajdla, T., Schiele, B., and Tuytelaars, T., editors, {\em Computer Vision -- ECCV 2014}, pages 740--755, Cham. Springer International Publishing.

\bibitem[Macchi, 1975]{Macchi1975}
Macchi, O. (1975).
\newblock The {{Coincidence Approach}} to {{Stochastic Point Processes}}.
\newblock {\em Advances in Applied Probability}, 7(1):83--122.

\bibitem[Murray et~al., 2010]{Murray2010}
Murray, I., Adams, R., and MacKay, D. (2010).
\newblock Elliptical slice sampling.
\newblock In {\em Proceedings of the {{Thirteenth International Conference}} on {{Artificial Intelligence}} and {{Statistics}}}, pages 541--548. {JMLR Workshop and Conference Proceedings}.

\bibitem[Newman, 2003]{Newman2003}
Newman, M. E.~J. (2003).
\newblock The structure and function of complex networks.
\newblock {\em SIAM Review}, 45(2):167--256.

\bibitem[Poinas, 2026]{poinas2026}
Poinas, A. (2026).
\newblock On determinantal point processes with nonsymmetric kernels.
\newblock {\em Electronic Journal of Probability}, 31:1--32.

\bibitem[Poulson, 2020]{poulson2020}
Poulson, J. (2020).
\newblock High-performance sampling of generic determinantal point processes.
\newblock {\em Philosophical Transactions of the Royal Society A: Mathematical, Physical and Engineering Sciences}, 378(2166):20190059.

\bibitem[Shah and Ghahramani, 2013]{Shah2013}
Shah, A. and Ghahramani, Z. (2013).
\newblock Determinantal clustering process - a nonparametric bayesian approach to kernel based semi-supervised clustering.
\newblock In {\em Proceedings of the Twenty-Ninth Conference on Uncertainty in Artificial Intelligence}, UAI'13, page 566^^e2^^80^^93575, Arlington, Virginia, USA. AUAI Press.

\bibitem[Shirai and Takahashi, 2003]{Shirai2003}
Shirai, T. and Takahashi, Y. (2003).
\newblock Random point fields associated with certain fredholm determinants i: fermion, poisson and boson point processes.
\newblock {\em Journal of Functional Analysis}, 205(2):414--463.

\bibitem[Valiant, 1979]{Valiant1979}
Valiant, L. (1979).
\newblock The complexity of computing the permanent.
\newblock {\em Theoretical Computer Science}, 8(2):189--201.

\bibitem[Van~Haaren and Davis, 2012]{Haaren2012}
Van~Haaren, J. and Davis, J. (2012).
\newblock Markov network structure learning: a randomized feature generation approach.
\newblock In {\em Proceedings of the Twenty-Sixth AAAI Conference on Artificial Intelligence}, AAAI'12, page 1148^^e2^^80^^931154. AAAI Press.

\bibitem[Zhang et~al., 2021]{Zhang2021}
Zhang, H., Juba, B., and Van Den~Broeck, G. (2021).
\newblock Probabilistic generating circuits.
\newblock In Meila, M. and Zhang, T., editors, {\em Proceedings of the 38th International Conference on Machine Learning}, volume 139 of {\em Proceedings of Machine Learning Research}, pages 12447--12457. PMLR.

\end{thebibliography}
\bibliographystyle{apalike}

\clearpage
\appendix
\thispagestyle{empty}

\onecolumn
\aistatstitle{Generalized Determinantal Point Processes for Attraction and Repulsion with Distinct Interaction Kernels: \\
Supplementary Materials}
\label{app:proof_thm}
\section{Proofs of theorems}
The pairwise interaction of a DRPP is expressed as the ratio of the correlations of the repulsive and attractive kernels as follows.
\begin{lemma}
Let $\bm{K},\bm{S}\succ0$, and let the correlations of $\bm{K}[\{i, j\}]$ and $\bm{S}[\{i, j\}]$ be
$\rho^K_{ij}=k_{ij}/\sqrt{k_{ii}k_{jj}}$ and $\rho^S_{ij}=s_{ij}/\sqrt{s_{ii}s_{jj}}$, respectively.
Then, for any $i\neq j$,
\begin{equation}
    \Gamma_{ij}=\frac{1-(\rho^K_{ij})^2}{1-(\rho^S_{ij})^2}.
\end{equation}
\label{lem:drpp_pairwise}
\end{lemma}
\begin{proof}
    By~\eqref{eq:pairwise_interaction}, $\Gamma_{ij}$ is determined only by the unnormalized weights $\tilde{P}$, and $\tilde{P}(\emptyset)=1$ by definition. From the definition of the DRPP~\eqref{eq:drpp}, the unnormalized weights are
    \begin{equation}
        \tilde{P}(\{i\})=\frac{k_{ii}}{s_{ii}},\qquad
        \tilde{P}(\{i,j\})=\frac{\det(\bm{K}[\{i,j\}])}{\det(\bm{S}[\{i,j\}])}=\frac{k_{ii}k_{jj}-k_{ij}^2}{s_{ii}s_{jj}-s_{ij}^2}.
    \end{equation}
    Here, $k_{ii},s_{ii}>0$ since $\bm{K},\bm{S}\succ0$, and $s_{ii}s_{jj}-s_{ij}^2>0$ since $\bm{S}[\{i,j\}]\succ0$. Thus, we obtain
    \begin{equation}
        \Gamma_{ij}
        =\frac{\tilde{P}(\{i,j\})\,\tilde{P}(\emptyset)}{\tilde{P}(\{i\})\,\tilde{P}(\{j\})}
        =\frac{k_{ii}k_{jj}-k_{ij}^2}{s_{ii}s_{jj}-s_{ij}^2}\cdot\frac{s_{ii}s_{jj}}{k_{ii}k_{jj}}
        =\frac{1-k_{ij}^2/(k_{ii}k_{jj})}{1-s_{ij}^2/(s_{ii}s_{jj})}
        =\frac{1-(\rho^K_{ij})^2}{1-(\rho^S_{ij})^2}.
    \end{equation}
\end{proof}

By the lemma, $\Gamma_{ij}$ becomes smaller as the items $i$ and $j$ become more similar in the sense of $\bm{K}$, and larger as they become more similar in the sense of $\bm{S}$.
That is, the contributions of repulsion and attraction appear separately as a ratio.
In particular, in the standard form, $\bm{K}=\bm{Q}\bm{R}_K\bm{Q}$ with $\mathrm{diag}(\bm{R}_K)=\boldsymbol{1}$, so $\bm{Q}$ does not affect $\Gamma_{ij}$.
In other words, the quality changes how likely each item is to be selected but does not affect the interactions between items, consistent with the insight in \citep{anzai2026}.
Moreover, since $\rho^K_{ij}=\langle \bm{v}_i,\bm{v}_j\rangle$ and $\rho^S_{ij}=\langle \bm{u}_i,\bm{u}_j\rangle$ for $i \neq j$, we can write
\begin{equation}
    \Gamma_{ij}=\frac{1-\langle \bm{v}_i,\bm{v}_j\rangle^2}{1-\langle \bm{u}_i,\bm{u}_j\rangle^2}.
\end{equation}

On the other hand, the pairwise interaction of an NDPP can be expressed in terms of $\bm{H}$ and $\bm{A}$ as follows.
\begin{lemma}
Let $\bm{H}^\top=\bm{H}$ and $\bm{A}^\top=-\bm{A}$, and suppose that $H_{ii}>0$ for all $i\in\Omega$. Then, for an NDPP and any $i\neq j$,
\begin{equation}
    \Gamma_{ij}=1 - \frac{H^2_{ij}-A^2_{ij}}{H_{ii}H_{jj}}.
\end{equation}
\label{lem:ndpp_pairwise}
\end{lemma}
\begin{proof}
    The unnormalized weight of an NDPP is $\tilde{P}(Y)=\det(\bm{L}[Y])$ for $Y\subseteq\Omega$ with $\det(\bm{L}[\emptyset])=1$, where $\bm{L}=\bm{H}+\bm{A}$.
    Since $\bm{A}^\top=-\bm{A}$, we have $A_{ii}=0$ and $A_{ji}=-A_{ij}$, and since $\bm{H}^\top=\bm{H}$, we have $H_{ji}=H_{ij}$. Therefore, $L_{ii}=H_{ii}$ and
    \begin{equation}
        \bm{L}[\{i,j\}]=
        \begin{pmatrix}
            H_{ii} & H_{ij}+A_{ij}\\
            H_{ij}-A_{ij} & H_{jj}
        \end{pmatrix}.
    \end{equation}
    Hence,
    \begin{equation}
        \tilde{P}(\{i\})=H_{ii},\qquad
        \tilde{P}(\{i,j\})=H_{ii}H_{jj}-(H_{ij}+A_{ij})(H_{ij}-A_{ij})=H_{ii}H_{jj}-H_{ij}^2+A_{ij}^2.
    \end{equation}
    Substituting these into~\eqref{eq:pairwise_interaction} and using $H_{ii}>0$, we obtain
    \begin{equation}
        \Gamma_{ij}=\frac{H_{ii}H_{jj}-H_{ij}^2+A_{ij}^2}{H_{ii}H_{jj}}
        =1-\frac{H_{ij}^2-A_{ij}^2}{H_{ii}H_{jj}}.
    \end{equation}
\end{proof}

\subsection{Proof of Theorem~\ref{thm:block}}
Let $\bm{Z}\in \{0, 1\}^{N \times C}$ be the cluster membership matrix, that is, $Z_{ic}=\mathds{1}\{i \in \Omega_c\}$. We also define $\kappa=\sqrt{1-\gamma_-}$ and $\sigma=\sqrt{1-\gamma_-/\gamma_+}$, and set
\begin{align}
    \bm{K} =& (1-\kappa) \bm{I} + \kappa \mathbf{1}\mathbf{1}^\top, \\
    \bm{S} =& (1-\sigma) \bm{I} + \sigma \bm{Z} \bm{Z}^\top.
\end{align}
Then all the diagonal entries of $\bm{K}$ and $\bm{S}$ are equal to one. For $i$ and $j$ in different clusters, we have $K_{ij}=\kappa$ and $S_{ij}=0$. Denoting the pairwise interaction of the DRPP by $\Gamma^{\mathrm{DRPP}}_{ij}$, we obtain
\begin{equation}
    \Gamma^{\mathrm{DRPP}}_{ij}=1 - \kappa^2 = \gamma_-.
\end{equation}
On the other hand, for $i$ and $j$ in the same cluster, we have $K_{ij}=\kappa$ and $S_{ij}=\sigma$, which yields
\begin{equation}
    \Gamma^{\mathrm{DRPP}}_{ij}=\frac{1 - \kappa^2}{1 - \sigma^2} = \frac{\gamma_-}{\gamma_- / \gamma_+} = \gamma_+.
\end{equation}
Hence, $\Gamma^\mathrm{DRPP}_{ij}=\Gamma^\ast_{ij}$ holds exactly. In this case, $\bm{S}$ is in the standard form with $\epsilon_S=1-\sigma$ and $\bm{U}=\sqrt{\sigma}\bm{Z}\in\mathbb{R}^{N\times C}$, that is, $r_S=C$.

Denoting the pairwise interaction of the NDPP by $\Gamma^{\mathrm{NDPP}}_{ij}$, since $H_{ii}=1$ and $H_{ij}=\kappa$ for the NDPP, Lemma~\ref{lem:ndpp_pairwise} gives
\begin{equation}
    \Gamma^{\mathrm{NDPP}}_{ij}=1-\kappa^2+A^2_{ij}=\gamma_- + A^2_{ij}.
\end{equation}
Then, for $i$ and $j$ in different clusters, $\Gamma^\ast_{ij}=\gamma_-$ implies $A_{ij}=0$. For $i$ and $j$ in the same cluster, $\Gamma^\ast_{ij}=\gamma_+$ implies
\begin{equation}
    |A_{ij}|=a_0, \qquad a_0\coloneq\sqrt{\gamma_+-\gamma_-}>0.
\end{equation}
Therefore, $\bm{A}$ is a block-diagonal skew-symmetric matrix with one block per cluster, and each block $\bm{A}_c\coloneq\bm{A}[\Omega_c]$ satisfies
\begin{equation}
    |(\bm{A}_c)_{ij}| = a_0 \qquad (i \neq j).
\end{equation}

Let $m\coloneq m_c$ and $\bm{T}\coloneq\bm{A}_c/a_0$. Then $\bm{T}$ is skew-symmetric, and all its off-diagonal entries are $\pm1$. Suppose that $m$ is even. For a skew-symmetric matrix of even order,
\begin{equation}
    \det\bm{T}=\mathrm{Pf}(\bm{T})^2
\end{equation}
holds. The Pfaffian $\mathrm{Pf}(\bm{T})$ is a sum over all perfect matchings of $m$ vertices, i.e., all partitions of $\{1,\dots,m\}$ into $m/2$ disjoint pairs, and each term of the sum is $\pm1$. The number of perfect matchings of $m$ vertices is $(m-1)!!=1\cdot3\cdots(m-1)$, which is odd as a product of odd numbers. Thus, $\mathrm{Pf}(\bm{T})$ is a sum of an odd number of $\pm1$'s and cannot be even, so that $\mathrm{Pf}(\bm{T})\neq0$. Therefore, $\det\bm{T}\neq0$, that is,
\begin{equation}
    \mathrm{rank}(\bm{A}_c)=m\qquad(\text{$m$ even}).
\end{equation}

Suppose next that $m$ is odd. The $(m-1)\times(m-1)$ principal submatrix obtained by removing any one item has the same form as above, and since $m-1$ is even, it has full rank. Hence, $\mathrm{rank}(\bm{A}_c)\geq m-1$. On the other hand, a skew-symmetric matrix of odd order is always singular, and its rank is even, so it is at most $m-1$. Therefore,
\begin{equation}
    \mathrm{rank}(\bm{A}_c)=m-1\qquad(\text{$m$ odd}).
\end{equation}
Therefore, $\mathrm{rank}(\bm{A}_c)=2\lfloor m_c/2\rfloor$ in both cases. Since $\bm{A}$ is block diagonal,
\begin{equation}
    \mathrm{rank}(\bm{A})=\sum_{c=1}^{C}\mathrm{rank}(\bm{A}_c)=2\sum_{c=1}^{C}\left\lfloor\frac{m_c}{2}\right\rfloor .
\end{equation}

Since $\mathrm{rank}(\bm{A})\leq 2 r_A$, we have
\begin{equation}
    2 r_A \geq 2\sum_{c=1}^{C}\left\lfloor\frac{m_c}{2}\right\rfloor,
\end{equation}
that is, $r_A\geq\sum_{c=1}^{C}\lfloor m_c/2\rfloor$. Moreover, since $\lfloor m_c/2\rfloor\geq (m_c - 1)/2$, we obtain
\begin{equation}
    r_A \geq \sum_{c=1}^{C}\frac{m_c-1}{2}=\frac{N-C}{2}.
\end{equation}

\subsection{Proof of Theorem~\ref{thm:skew}}
\textup{(i)} Since $\bm{L}+\bm{L}^\top=2\bm{I}\succ0$, $\bm{L}$ is an admissible $P_0$ matrix~\citep{gartrell2019a}, and $H_{ii}=1$ and $H_{ij}=0$, Lemma~\ref{lem:ndpp_pairwise} gives $\Gamma_{ij}=1+A_{ij}^2$. Therefore, setting
\begin{align}
    |A_{ij}|=&\sqrt{\gamma_{ij}-1}\ \ ((i,j)\in E), \\
    A_{ij}=&0\ \ ((i,j)\notin E),\qquad A_{ji}=-A_{ij},
\end{align}
we obtain $\Gamma_{ij}=\gamma_{ij}$ on the edges and $\Gamma_{ij}=1$ on the non-edges.

\textup{(ii)}
By Lemma~\ref{lem:drpp_pairwise}, $\Gamma_{ij}=(1-(\rho^K_{ij})^2)/(1-(\rho^S_{ij})^2)\geq\gamma$. Since $1-(\rho^K_{ij})^2\leq1$, we have $1-(\rho^S_{ij})^2\leq1/\gamma$, and since $\rho^S_{ij}\geq0$, we obtain
\begin{equation}
    \rho^S_{ij}\geq r,\qquad r\coloneq\sqrt{1-\frac{1}{\gamma}}.
\end{equation}
Similarly, $\rho^S_{jk}\geq r$. Since $\bm{S}\succ0$, its principal submatrices are also positive definite, and hence the determinant of the correlation matrix corresponding to $\{i,j,k\}$,
\begin{equation}
\det\bm{R}_S[\{i,j,k\}]=
\begin{vmatrix}
1 & \rho^S_{ij} & \rho^S_{ik}\\
\rho^S_{ij} & 1 & \rho^S_{jk}\\
\rho^S_{ik} & \rho^S_{jk} & 1
\end{vmatrix}
\end{equation}
is nonnegative. Rearranging this condition with respect to $\rho^S_{ik}$ gives $(\rho^S_{ik}-\rho^S_{ij}\rho^S_{jk})^2\leq(1-(\rho^S_{ij})^2)(1-(\rho^S_{jk})^2)$, and thus
\begin{equation}
    \rho^S_{ik}\geq \rho^S_{ij}\rho^S_{jk}-\sqrt{(1-(\rho^S_{ij})^2)(1-(\rho^S_{jk})^2)}.
\end{equation}
Let $\rho^S_{ij}=\cos\alpha$ and $\rho^S_{jk}=\cos\beta$. Since $\gamma>2$ implies $r>1/\sqrt2$, we have $0\leq\alpha,\beta\leq\arccos r<\pi/4$, and hence $\alpha+\beta<\pi/2$. Therefore,
\begin{equation}
    \rho^S_{ik}\geq\cos\alpha\cos\beta-\sin\alpha\sin\beta=\cos(\alpha+\beta)
    \geq\cos(2\arccos r)=2r^2-1=1-\frac{2}{\gamma}.
\end{equation}
Since $1-2/\gamma>0$ for $\gamma>2$,
\begin{equation}
    1-(\rho^S_{ik})^2\leq1-\left(1-\frac{2}{\gamma}\right)^2=\frac{4(\gamma-1)}{\gamma^2}.
\end{equation}
Moreover, $|\rho^K_{ik}|\leq\delta$ implies $1-(\rho^K_{ik})^2\geq1-\delta^2$, and applying Lemma~\ref{lem:drpp_pairwise} again yields~\eqref{eq:endpoint_bound}. 
Finally, realizing $\Gamma_{ik}\leq1$ requires $\frac{\gamma^2}{4(\gamma-1)}(1-\delta^2)\leq1$, which is equivalent to $\delta^2\geq(1-2/\gamma)^2$. Since $\delta\geq0$, we obtain~\eqref{eq:cancellation}.

\section{Details of learning DRPPs}
\label{app:detail_algorithm}
Throughout this section, we write $\bm{z}=(\bm{z}_1,\bm{z}_2)$ for simplicity and use the notation $\tilde{\bm{L}}(\bm{z})\coloneq\tilde{\bm{L}}(\bm{z}_1,\bm{z}_2)$, $\bm{C}(\bm{z})\coloneq\bm{C}(\bm{z}_1,\bm{z}_2)$, $P(\bm{z})\coloneq P(\bm{z}_1,\bm{z}_2)$, and $P_k(\bm{z})\coloneq P(\bm{z}_1,\bm{z}_2;k)$.

\begin{algorithm}[t]
  \caption{Maximum likelihood estimation of DRPP}
  \label{alg:app_mle}
  \KwIn{Observations $\{A_m\}_{m=1}^{M}$, ranks $r_S, r_K$, $\epsilon_S, \epsilon_K$, size $k$ (optionally), number of chains $T$, minibatch size $B$}
  \KwOut{$\boldsymbol{\theta} = (\boldsymbol{\alpha}, \bm{V}, \bm{U})$}
  Initialize $\boldsymbol{\theta}$ and chains $\{\bm{z}^{(t)}\}_{t=1}^{T}$\;
  \Repeat{convergence}{
    \For{$t = 1$ \KwTo $T$}{
      \eIf{$k$ specified}{
        Update $\bm{z}^{(t)}$ by ESS targeting $P_k(\bm{z})$ \eqref{eq:Pz_k}\;
      }{
        Update $\bm{z}^{(t)}$ by ESS targeting $P(\bm{z})$ \eqref{eq:Pz}\;
      }
    }
    Sample a minibatch $\mathcal{B}\subseteq\{1,\dots,M\}$ with $|\mathcal{B}|=B$\;
    Compute $\hat g(\boldsymbol{\theta})$ by \eqref{eq:app_grad_est}\;
    Update $\boldsymbol{\theta}$ with $\hat g(\boldsymbol{\theta})$ by a stochastic gradient method\;
  }
\end{algorithm}

\subsection{Learning Algorithm}
\label{app:grad}
We maximize~\eqref{eq:mle} for data with variable subset sizes, and its size-$k$ conditional counterpart, obtained by replacing $Z(\boldsymbol{\theta})$ with $Z_k(\boldsymbol{\theta})$, for data in which every observation satisfies $|A_m|=k$. By~\eqref{eq:unnormalized_score}, the first term can be computed exactly from determinants of $|A_m|\times|A_m|$ matrices. On the other hand, by Corollary~\ref{cor:dZ} and Section~\ref{sec:fixed_cardinality}, the gradient of the normalizing-constant term is the expectation of $\nabla_{\boldsymbol{\theta}}\log w_{\boldsymbol{\theta}}(\bm{z})$ with respect to $P(\bm{z})$ or $P_k(\bm{z})$, where
\begin{equation}
  w_{\boldsymbol{\theta}}(\bm{z}) \coloneq
  \begin{cases}
    \det(\bm{I} + \tilde{\bm{L}}(\bm{z})) & \text{(variable size)},\\
    e_k(\tilde{\bm{L}}(\bm{z})) & \text{(fixed size)}.
  \end{cases}
  \label{eq:app_w}
\end{equation}
In our experiments, we therefore estimate the gradient using samples $\{\bm{z}^{(t)}\}_{t=1}^{T}$ obtained by ESS and a minibatch $\mathcal{B}\subseteq\{1,\dots,M\}$ of observations as
\begin{equation}
  \hat g(\boldsymbol{\theta})
  = \frac{1}{|\mathcal{B}|} \sum_{m\in\mathcal{B}} \nabla_{\boldsymbol{\theta}} \log \tilde{P}(A_m \mid \boldsymbol{\theta})
    - \frac{1}{T} \sum_{t=1}^{T} \nabla_{\boldsymbol{\theta}} \log w_{\boldsymbol{\theta}}(\bm{z}^{(t)})
    - \nabla_{\boldsymbol{\theta}} \mathcal{R}(\boldsymbol{\theta}),
  \label{eq:app_grad_est}
\end{equation}
where the samples $\bm{z}^{(t)}$ are treated as constants independent of $\boldsymbol{\theta}$. The regularization term $\mathcal{R}(\boldsymbol{\theta})$ is defined as
\begin{equation}
  \mathcal{R}(\boldsymbol{\theta})
  = \frac{\eta_\alpha}{2}\,\|\boldsymbol{\alpha}\|_2^2
  + \frac{\eta_U}{2}\sum_{a\neq b}\bigl(\bm{u}_a^\top\bm{u}_b\bigr)^2
  + \frac{\eta_V}{2}\sum_{a\neq b}\bigl(\bm{v}_a^\top\bm{v}_b\bigr)^2,
  \label{eq:app_reg}
\end{equation}
where $\eta_\alpha, \eta_U, \eta_V \geq 0$ are regularization weights. Algorithm~\ref{alg:app_mle} summarizes the overall procedure.

\subsection{Hyperparameter Selection via Held-out Likelihood}
\label{app:heldout}
We select the hyperparameters by maximizing the log-likelihood of the validation data $\{A_n\}_{n=1}^{n_\mathrm{val}}$. Since the normalizing constant of a DRPP is not available in closed form, we use the representation of the DRPP as an infinite mixture of DPPs given by Theorem~\ref{thm:mixture}. Because $P(A\mid\bm{z})=\det(\tilde{\bm{L}}(\bm{z})[A])/\det(\bm{I}+\tilde{\bm{L}}(\bm{z}))$, we have
\begin{equation}
  P(A)=\int P(A\mid\bm{z})\,P(\bm{z})\,d\bm{z}
  =\mathbb{E}_{\bm{z}\sim P(\bm{z})}\left[\frac{\det(\tilde{\bm{L}}(\bm{z})[A])}{\det(\bm{I}+\tilde{\bm{L}}(\bm{z}))}\right].
\end{equation}
Similarly, in the fixed-size case, $P(A\mid\bm{z};k)=\det(\tilde{\bm{L}}(\bm{z})[A])/e_k(\tilde{\bm{L}}(\bm{z}))$ yields
\begin{equation}
  P(A\mid|A|=k)
  =\mathbb{E}_{\bm{z}\sim P_k(\bm{z})}\left[\frac{\det(\tilde{\bm{L}}(\bm{z})[A])}{e_k(\tilde{\bm{L}}(\bm{z}))}\right].
\end{equation}
We therefore draw $\{\bm{z}^{(t)}\}_{t=1}^{T}$ from $P(\bm{z})$ or $P_k(\bm{z})$ by ESS and, using $w_{\boldsymbol{\theta}}$ in~\eqref{eq:app_w}, estimate the log-likelihood of the validation data as
\begin{equation}
  \sum_{n=1}^{n_\mathrm{val}}\log P(A_n)
  \approx\sum_{n=1}^{n_\mathrm{val}}\log\left\{\frac{1}{T}\sum_{t=1}^{T}\frac{\det(\tilde{\bm{L}}(\bm{z}^{(t)})[A_n])}{w_{\boldsymbol{\theta}}(\bm{z}^{(t)})}\right\}.
  \label{eq:app_heldout}
\end{equation}
This estimator does not require the normalizing constant, and each term can be computed from the determinant of an $|A_n|\times|A_n|$ matrix and $w_{\boldsymbol{\theta}}(\bm{z}^{(t)})$ alone.

\subsection{Efficient Evaluation}
\label{app:lowrank}
Since the shrinkage loop of ESS evaluates $w_{\boldsymbol{\theta}}(\bm{z})$ at every acceptance test, the evaluation of $w_{\boldsymbol{\theta}}(\bm{z})$ dominates the overall computational cost of learning.
When $\bm{R}_K = \bm{I}$, $\bm{L}$ and $\tilde{\bm{L}}(\bm{z})$ are diagonal; $\det(\bm{I} + \tilde{\bm{L}})$ can be computed in $O(N)$ time, and $e_k(\tilde{\bm{L}})$ in $O(Nk)$ time by the standard recursion over the diagonal entries.
When $r_K > 0$, in contrast, $\tilde{\bm{L}}(\bm{z})$ is dense, and a naive computation requires an $O(N^3)$ decomposition.
We therefore exploit the low-rank structure in~\eqref{eq:L_tilde}.
Substituting $\bm{L}=\epsilon_S^{-1}\bm{Q}\bm{R}_K\bm{Q}$ and $\bm{Q}=\mathrm{diag}(e^{\boldsymbol{\alpha}/2})$ into~\eqref{eq:L_tilde}, we can write
\begin{equation}
  \tilde{\bm{L}}(\bm{z}) = \boldsymbol{\Lambda}_{\bm{z}} \bm{R}_K \boldsymbol{\Lambda}_{\bm{z}},
  \qquad
  \boldsymbol{\Lambda}_{\bm{z}} \coloneq \epsilon_S^{-1/2}\,\bm{Q}\,\bm{C}(\bm{z})^{1/2}
  = \mathrm{diag}\bigl(\lambda_i(\bm{z})\bigr)_{i \in \Omega},
  \qquad
  \lambda_i(\bm{z})^2 = \frac{e^{\alpha_i}\, \bm{C}(\bm{z})_{ii}}{\epsilon_S}.
  \label{eq:app_Ltilde}
\end{equation}
Letting $\bm{W} \coloneq \boldsymbol{\Lambda}_{\bm{z}} \bm{V} \in \mathbb{R}^{N \times r_K}$, we decompose $\tilde{\bm{L}}(\bm{z})$ into the sum of a diagonal matrix and a low-rank matrix:
\begin{equation}
  \tilde{\bm{L}}(\bm{z}) = \epsilon_K \boldsymbol{\Lambda}_{\bm{z}}^2 + \bm{W} \bm{W}^\top.
\end{equation}
In the following, we abbreviate $\lambda_i = \lambda_i(\bm{z})$ and denote the $i$th row of $\bm{V}$ by $\bm{v}_i$.

\paragraph{Variable size.}
By the matrix determinant lemma,
\begin{equation}
  \det(\bm{I} + \tilde{\bm{L}})
  = \det(\bm{I} + \epsilon_K \boldsymbol{\Lambda}_{\bm{z}}^2)\,
    \det\left(\bm{I}_{r_K} + \bm{W}^\top (\bm{I} + \epsilon_K \boldsymbol{\Lambda}_{\bm{z}}^2)^{-1} \bm{W}\right).
\end{equation}
The first factor is the determinant of a diagonal matrix and can be computed in $O(N)$ time.
Since the $r_K \times r_K$ matrix in the second factor can be written as
\begin{equation}
  \bm{W}^\top (\bm{I} + \epsilon_K \boldsymbol{\Lambda}_{\bm{z}}^2)^{-1} \bm{W}
  = \sum_{i=1}^{N} \frac{\lambda_i^2}{1 + \epsilon_K \lambda_i^2} \bm{v}_i \bm{v}_i^\top,
\end{equation}
the overall cost is $O(N r_K^2 + r_K^3)$.

\paragraph{Fixed size.}
Although $e_k$ is a symmetric function of the eigenvalues, we avoid the eigendecomposition by using Newton's identities:
\begin{equation}
  k e_k = \sum_{j=1}^{k} (-1)^{j-1} e_{k-j} p_j,
  \qquad e_0 = 1.
\end{equation}
where $p_j\coloneq\operatorname{tr}(\tilde{\bm{L}}^j)$ is the $j$th power sum of the eigenvalues of $\tilde{\bm{L}}$. For example, for $k = 3$,
\begin{equation}
  e_1 = p_1, \qquad
  e_2 = \tfrac{1}{2}\bigl(p_1^2 - p_2\bigr), \qquad
  e_3 = \tfrac{1}{6}\bigl(p_1^3 - 3 p_1 p_2 + 2 p_3\bigr).
\end{equation}
Furthermore, letting $\bm{Y}_a \coloneq \bm{W}^\top \boldsymbol{\Lambda}_{\bm{z}}^{2a} \bm{W} \in \mathbb{R}^{r_K \times r_K}$ and
$\mu_a \coloneq \operatorname{tr}(\boldsymbol{\Lambda}_{\bm{z}}^{2a}) = \sum_i \lambda_i^{2a}$,
expanding $(\epsilon_K \boldsymbol{\Lambda}_{\bm{z}}^2 + \bm{W}\bm{W}^\top)^j$ and using the cyclic property of the trace yield
\begin{align}
  p_1 &= \epsilon_K \mu_1 + \operatorname{tr} \bm{Y}_0, \\
  p_2 &= \epsilon_K^2 \mu_2 + 2\epsilon_K \operatorname{tr} \bm{Y}_1 + \operatorname{tr}(\bm{Y}_0^2), \\
  p_3 &= \epsilon_K^3 \mu_3 + 3\epsilon_K^2 \operatorname{tr} \bm{Y}_2
         + 3\epsilon_K \operatorname{tr}(\bm{Y}_1 \bm{Y}_0) + \operatorname{tr}(\bm{Y}_0^3).
\end{align}
Computing $\bm{Y}_0, \dots, \bm{Y}_{k-1}$ requires $O(k N r_K^2)$ time, and the remaining operations involve only $r_K \times r_K$ matrices. Hence, the evaluation avoids both the combinatorial cost of summing over all $\binom{N}{k}$ subsets and the $O(N^3)$ cost of eigendecomposition. Note, however, that Newton's identities involve alternating sums, so numerical errors due to cancellation grow as $k$ increases.

\section{Details of Experimental Setting}
\label{app:detail_exp}

In this section, we provide details of the synthetic datasets.

\subsection{Synthetic Datasets}
We defined the pairwise interactions of the artificial datasets as follows, where $\theta_i = \frac{\pi}{2}\cdot\frac{i-1}{N-1}$, $c_i \in \{1,\dots,C\}$ denotes the group to which item $i$ belongs, and $\varphi_{c}=2\pi c/C$:
\begin{align}
  \Gamma^{\mathrm{cont}}_{ij} &= \bigl(1 - (1-\varepsilon)^2\cos^2(\theta_i-\theta_j)\bigr)^{-1},
  \\
  \Gamma^{\mathrm{clust}}_{ij} &=
  \begin{cases}
    \bigl(1 - (1-\varepsilon)^2\bigr)^{-1} & (c_i = c_j),\\
    1 & (c_i \ne c_j),
  \end{cases}
  \\
  \Gamma^{\mathrm{skew}}_{ij} &= 1 + a^2\sin^2(\varphi_{c_j}-\varphi_{c_i}).
\end{align}
We set $\varepsilon=0.05$ and $a=2.0$, and used $C=6$ groups for cluster and $C=3$ groups for skew.
For $C=3$, $\sin^2(\varphi_{c_j}-\varphi_{c_i})=3/4$ for any pair of different groups.

The data were generated from the pairwise log-linear model with these pairwise interactions,
\begin{equation}
    p(A \mid |A|=k) \propto \prod_{\{i,j\} \subseteq A}\Gamma_{ij}
    = \exp\Bigl\{\sum_{\{i,j\} \subseteq A}\log \Gamma_{ij}\Bigr\}.
\end{equation}
Since the model before conditioning on the size, $\tilde{p}(A)=\prod_{\{i,j\}\subseteq A}\Gamma_{ij}$, satisfies $\tilde{p}(\emptyset)=\tilde{p}(\{i\})=1$, its pairwise interaction defined in~\eqref{eq:pairwise_interaction} coincides exactly with $\Gamma_{ij}$. Moreover, since this model has no interactions of order three or higher, it prevents any particular method from gaining an advantage through structures other than the pairwise interactions. We fixed the size of the generated subsets to $k=3$. This avoids the generated subsets being polarized toward the empty set and the full set, and keeps the normalizing constant exactly computable even for $N=60$, which enables both exact data generation and exact computation of the negative log-likelihood (NLL) on the validation data. For all synthetic datasets, we generated $n_{\mathrm{train}}=10000$ samples and $n_{\mathrm{val}}=10000$  validation samples.

Among these, the pairwise interactions of continuous and cluster can be represented exactly by the attractive kernel of a DRPP.
When $\bm{R}_K=\bm{I}$, Lemma~\ref{lem:drpp_pairwise} gives
\begin{equation}
    \Gamma_{ij} = \frac{1}{1-(\rho^S_{ij})^2}= \frac{1}{1-\langle \bm{u}_i,\bm{u}_j \rangle^2}.
\end{equation}
For continuous, letting
\begin{equation}
  \bm{u}_i = \sqrt{1-\varepsilon}\,(\cos\theta_i,\ \sin\theta_i),
\end{equation}
we have $\langle\bm{u}_i,\bm{u}_j\rangle=(1-\varepsilon)(\cos\theta_i\cos\theta_j+\sin\theta_i\sin\theta_j)=(1-\varepsilon)\cos(\theta_i-\theta_j)$, and hence
\begin{equation}
    \Gamma_{ij}=\bigl(1 - (1-\varepsilon)^2\cos^2(\theta_i-\theta_j)\bigr)^{-1}=\Gamma^{\mathrm{cont}}_{ij}.
\end{equation}
For cluster, letting
\begin{equation}
  \bm{u}_i = \sqrt{1-\varepsilon}\,\bm{e}_{c_i},
\end{equation}
where $\bm{e}_c\in\mathbb{R}^C$ is the $c$-th standard basis vector, $\langle\bm{u}_i,\bm{u}_j\rangle$ equals $1-\varepsilon$ for items in the same group and $0$ otherwise, and hence $\Gamma_{ij}=\Gamma^{\mathrm{clust}}_{ij}$.
In both cases, $\|\bm{u}_i\|_2^2=1-\varepsilon$, so these constructions are in the standard form with $\epsilon_S=\varepsilon$.
Therefore, a DRPP can represent continuous exactly with $r_S=2$ and cluster exactly with $r_S=C$.
We restricted $\theta_i$ to a quarter circle rather than the full circle because $\Gamma_{ij}$ depends only on $\cos^2(\theta_i-\theta_j)$.
For $|\theta_i-\theta_j|\leq\pi/2$, $\Gamma_{ij}$ decreases monotonically with the angular difference, but it increases again once the difference exceeds $\pi/2$, which would make distant items attract each other.

On the other hand, skew can be represented exactly by the skew-symmetric part of an NDPP.
When $\bm{H}=\bm{I}$, Lemma~\ref{lem:ndpp_pairwise} gives
\begin{equation}
    \Gamma_{ij}=1+A^2_{ij}.
\end{equation}
We associate each group $c$ with $\bm{p}_c = (\cos\varphi_c,\ \sin\varphi_c)$ and let
\begin{equation}
  \bm{A} = \bm{P}\bm{J}\bm{P}^\top,
\end{equation}
where $\bm{P} \in \mathbb{R}^{N \times 2}$ is the matrix whose $i$th row is $\sqrt{a}\,\bm{p}_{c_i}$ and $\bm{J} = \begin{psmallmatrix} 0 & 1 \\ -1 & 0\end{psmallmatrix}$.
Then $\bm{A}$ is skew-symmetric and $A_{ij} = a \sin(\varphi_{c_j} - \varphi_{c_i})$, so that $\Gamma_{ij}=\Gamma^{\mathrm{skew}}_{ij}$.
Letting $\bm{B}$ and $\bm{C}$ be the first and second columns of $\bm{P}$, respectively, we have $\bm{A}=\bm{B}\bm{C}^\top-\bm{C}\bm{B}^\top$; hence, an NDPP can represent skew exactly with $r_A=1$.

\subsection{Real Datasets}

We use three real datasets: Amazon Baby Registries~\citep{gartrell2019a}, KDD Cup 2000 from the Twenty Datasets~\citep{Haaren2012}, and Microsoft COCO~\citep{Lin2014}. Since our focus is on interactions between items, we remove subsets with fewer than two items from every dataset; such subsets carry no information about item interactions and would otherwise dominate the likelihood (they account for more than 90\% of the raw KDD data, whereas the Amazon data contain none). The same preprocessing is also adopted by \citet{Gartrell2016} and others, albeit for different purposes.

We use the distributed training, validation, and test splits for KDD, and a random $70/15/15$ split with a fixed seed for Amazon and COCO, which have no distributed splits. Since our goal is to analyze how the NLL changes with the number of parameters as the hyperparameters vary, we fit the parameters on the training set and evaluate every hyperparameter setting by its NLL on the validation set; the test sets are not used for model fitting or NLL evaluation.. To reduce the cost of the ESS-based likelihood of DRPP, we use the first $1000$ validation subsets of each dataset. Table~\ref{tab:real_data} shows the summary of the real datasets.

\paragraph{Measuring Transitivity}
For real datasets, we measure transitivity on the graph whose edges connect pairs of items with attraction. For each pair of items $(i,j)$, let $n_{11}$, $n_{10}$, $n_{01}$, and $n_{00}$ denote the numbers of subsets that contain both $i$ and $j$, only $i$, only $j$, and neither, respectively, counted over the whole dataset (training, validation, and test subsets after removing subsets with fewer than two items). We compute the log odds ratio with the Haldane--Anscombe correction~\citep{Anscombe1956,Haldane1956},
\begin{equation}
    \tilde{\Psi}_{ij}=\log\frac{(n_{11}+1/2)(n_{00}+1/2)}{(n_{10}+1/2)(n_{01}+1/2)},
\end{equation}
where adding $1/2$ to each cell keeps $\tilde{\Psi}_{ij}$ finite even when a cell is empty, and define the adjacency matrix $\bm{G}$ by $G_{ij}=\mathbbm{1}\{\tilde{\Psi}_{ij}>0\}$ for $i\neq j$ and $G_{ii}=0$. Letting $d_i=\sum_j G_{ij}$ be the degree of item $i$, we measure the transitivity of $\bm{G}$ by the clustering coefficient~\citep{Newman2003}
\begin{equation}
    T_\mathrm{trans}(\bm{G})=\frac{\mathrm{Tr}(\bm{G}^3)}{2W},\qquad W=\sum_i\binom{d_i}{2},
\end{equation}
that is, the fraction of paths of length two that are closed into triangles. Since $T_\mathrm{trans}$ increases with the density and the degree heterogeneity of $\bm{G}$, we use its difference from the average over $B$ degree-preserving random graphs with adjacency matrices $\bm{G}^\mathrm{null}_b$ ($b=1,\dots,B$):
\begin{equation}
    T_\mathrm{excess}(\bm{G})=T_\mathrm{trans}(\bm{G})-\frac{1}{B}\sum_{b=1}^{B}T_\mathrm{trans}(\bm{G}^\mathrm{null}_b).
\end{equation}
Each $\bm{G}^\mathrm{null}_b$ is generated from $\bm{G}$ by double-edge swaps: two edges $(i,j)$ and $(k,l)$ are chosen uniformly at random and replaced by $(i,l)$ and $(k,j)$ when the four items are distinct and neither new pair is already an edge. This rewiring preserves every degree $d_i$ and creates neither self-loops nor multiple edges; we repeat it until $20|E|$ swaps have succeeded, where $|E|$ is the number of edges of $\bm{G}$, and use $B=100$. $T_\mathrm{excess}=0$ indicates that $\bm{G}$ has no transitivity beyond what its degrees explain. The values of $T_\mathrm{excess}$ for all datasets are listed in Table~\ref{tab:real_data}.

\begin{table*}[t]
\centering
\caption{Real datasets: numbers of subsets before (raw) and after removing subsets with fewer than two items, and transitivity. }
\label{tab:real_data}
\begin{tabular}{lrlrrrrrrr}
\toprule
 & & & \multicolumn{3}{c}{Raw} & \multicolumn{3}{c}{$|A|\ge 2$} & \\
\cmidrule(lr){4-6}\cmidrule(lr){7-9}
Dataset & $N$ & Split & Train & Valid & Test & Train & Valid & Test & $T_\mathrm{excess}$ \\
\midrule
Amazon & 300 & 70/15/15 & 21{,}852 & 4{,}682 & 4{,}684 & 21{,}852 & 4{,}682 & 4{,}684 & $+0.582$ \\
COCO & 80 & 70/15/15 & 82{,}800 & 17{,}743 & 17{,}744 & 65{,}155 & 13{,}962 & 13{,}963 & $+0.290$ \\
KDD & 64 & given & 180{,}092 & 19{,}907 & 34{,}955 & 15{,}841 & 1{,}778 & 2{,}668 & $+0.208$ \\
\bottomrule
\end{tabular}
\end{table*}

\subsection{Hyperparameter Setting}
We evaluate the relationship between the number of parameters and the validation NLL of each method. To this end, we vary only the hyperparameters that determine the number of parameters, namely the ranks of the low-rank factors for DRPPs, NDPPs, and DPPs, and the maximum group size (denoted by $K$) and the number of components (denoted by $C$) for SimplePGC, as listed in Table~\ref{tab:hyper}. 

All other hyperparameters are fixed. For DRPPs, we set $\epsilon_S=\epsilon_K=0.03$ and the regularization weights in $R(\theta)$ to $\eta_\alpha=\eta_U=\eta_V=10^{-3}$. For the baselines, we used their original regularization terms with the same weight $10^{-3}$. For SimplePGC, we used the parameters at the epoch with the best validation NLL.

\begin{table*}[t]
\centering
\caption{Hyperparameter candidates. Synthetic: Cluster, Continuous, Skew; Real: Amazon, KDD, COCO.}
\label{tab:hyper}
\begin{tabular}{llll}
\toprule
Method & Hyperparameter & Synthetic & Real \\
\midrule
DRPP-attr & $r_S$ & 2, 3, 4, 5, 6, 7, 8 & 2, 3, 5, 10, 20, 30 \\
DRPP & $r_S$ & 2, 3, 4, 5, 6, 7, 8 & 2, 3, 5, 10, 20, 30 \\
 
NDPP-attr & $r_A$ & 1, 3, 5, 8, 15, 20, 30 & 2, 3, 5, 10, 20, 30 \\
NDPP & $r_H$ & 1, 2, 3, 4, 5, 6 & 2, 3, 5, 10, 20, 30 \\
 & $r_A$ & 1, 3, 5, 8, 15, 20, 30 & 2, 3, 5, 10, 20, 30 \\
DPP & $r$ & 1, 2, 4, 6, 8, 10, 15, 20, 30, 40 & 2, 3, 5, 10, 20, 30, 50 \\
SimplePGC & $K$ & 1, 2, 4, 6, 7, 10 & 1, 2, 4, 6, 8, 10 \\
 & $C$ & 1, 2, 4, 8, 10 & 1, 2, 4, 6, 10, 20 \\
\bottomrule
\end{tabular}
\end{table*}

 \subsection{Implementation of Baselines}
We reimplemented the baselines in JAX: SimplePGC based on the official implementation of \citet{Zhang2021} (\url{https://github.com/joshuacnf/Probabilistic-Generating-Circuits}) and NDPP following the official implementation of \citet{gartrell2019a} (\url{https://github.com/cgartrel/nonsymmetric-DPP-learning}). The models and training objectives follow the originals; our changes only speed up the computation, except that early stopping of SimplePGC uses the same validation set as the other methods instead of a held-out part of the training set.

\end{document}